\PassOptionsToPackage{table}{xcolor}

\documentclass[10pt]{article}
\usepackage[accepted]{tmlr}

\usepackage[utf8]{inputenc}
\usepackage[T1]{fontenc}

\usepackage{amsmath,amssymb,amsfonts,amsthm}
\newtheorem{theorem}{Theorem}
\newtheorem{lemma}[theorem]{Lemma}

\usepackage{graphicx}
\usepackage{booktabs}
\usepackage{multirow}
\usepackage{siunitx}
\usepackage{threeparttable}
\usepackage{makecell}
\usepackage{subcaption}
\usepackage{enumitem}
\usepackage{colortbl}
\usepackage{algorithm}
\usepackage{algpseudocode}
\usepackage{url}
\usepackage{hyperref}

\title{Efficient and Unbiased Sampling from Boltzmann
\\Distributions via Variance-Tuned Diffusion Models}

\author{\name{Fengzhe Zhang}\email{fz287@cam.ac.uk}\\\addr{University of Cambridge}
\AND
\name{Laurence I. Midgley}\email{laurence@angstrom-ai.com}\\\addr{Ångström AI}\\\addr{University of Cambridge}
\AND
\name{José Miguel Hernández-Lobato}\email{jmh233@cam.ac.uk}\\\addr{Ångström AI}\\\addr{University of Cambridge}}

\def\month{10}
\def\year{2025}
\def\openreview{\url{https://openreview.net/forum?id=Jq2dcMCS5R}}

\begin{document}

\maketitle

\begin{abstract}
Score-based diffusion models (SBDMs) are powerful amortized samplers for Boltzmann distributions; however, imperfect score estimates bias downstream Monte Carlo estimates. Classical importance sampling (IS) can correct this bias, but computing exact likelihoods requires solving the probability-flow ordinary differential equation (PF–ODE), a procedure that is prohibitively costly and scales poorly with dimensionality.
We introduce Variance-Tuned Diffusion Importance Sampling (VT-DIS), a post-training method that adapts the per-step noise covariance of a pretrained SBDM by minimizing the $\alpha$-divergence ($\alpha=2$) between its forward diffusion and reverse denoising trajectories. VT-DIS assigns a single trajectory-wise importance weight to the joint forward–reverse process, yielding unbiased expectation estimates at test time with negligible inference-time overhead compared to standard sampling.
On the DW-4, LJ-13, and alanine-dipeptide benchmarks, VT-DIS achieves effective sample sizes of approximately 80\%, 35\%, and 3.5\%, respectively, while using only a fraction of the computational budget required by vanilla diffusion + IS or PF-ODE–based IS. Our code is available at \url{https://github.com/fz920/cov_tuned_diffusion}.
\end{abstract}

\section{Introduction}
\label{section: introduction}

Sampling from Boltzmann distributions is central to many applications in statistical physics, chemistry, and materials science. Accurate equilibrium samples enable \emph{in silico} prediction of thermodynamic observables, accelerating the discovery of catalysts, drugs, and novel materials and reducing reliance on costly wet-lab experiments. However, the high dimensionality and multimodal nature of realistic energy landscapes render traditional Markov chain Monte Carlo (MCMC)~\citep{metropolis1953equation,hastings1970monte} and molecular dynamics (MD) simulations~\citep{frenkel2023understanding} prohibitively time-consuming for many practical problems.

\textbf{Normalizing-flow} (NF) Boltzmann generators~\citep{noe2019boltzmann,tan2025scalable,kohler2020equivariant} amortize sampling by learning an expressive, invertible mapping from a simple base distribution to the target, and exact density evaluation of generated samples permits unbiased importance sampling (IS). However, the requirement of tractable Jacobian determinants constrains architectural flexibility, becoming a bottleneck in high dimensions. Continuous normalizing flows (CNFs)~\citep{chen2018neural,grathwohl2019scalable} mitigate this issue by parameterizing the transformation as an ordinary differential equation (ODE) and have been combined with \(\mathrm{E}(3)\)-equivariant architectures for molecular systems~\citep{garcia2021n,klein2023equivariant}. Yet even CNFs demand expensive Jacobian computations for log-density evaluation.

\textbf{Score-based diffusion models} (SBDMs)~\citep{sohl2015deep,ho2020denoising,song2020score} offer an expressive alternative that does not require invertibility. As with any learned sampler, imperfect score estimates bias downstream Monte Carlo estimates. Two IS-based corrections exist for a pretrained score network: (i) solve the probability flow ODE to obtain exact likelihoods—an approach that is computationally intensive, and whose approximations reintroduce bias; or (ii) solve the reverse stochastic differential equation (SDE) and perform IS over the joint forward–reverse trajectory, whose density factorizes into Gaussian terms; however, naively applying this approach typically yields a very low effective sample size (ESS) because the joint proposal distribution cannot fully cover the target distribution.

We adopt the second approach and show that a simple post-training covariance adjustment makes it practical. We introduce \textbf{variance-tuned diffusion importance sampling} (VT-DIS), which freezes the score network and optimizes only the per-step noise covariance by minimizing the \(\alpha\)-divergence (\(\alpha = 2\)) between the forward diffusion and reverse sampling processes. VT-DIS yields an unbiased trajectory-wise estimator and achieves high ESS at an orders-of-magnitude lower computational cost than vanilla diffusion + IS or NF-based samplers. Our contributions are:
\begin{itemize}[left=0in]
\item \textbf{Algorithm.} We propose VT-DIS, a post-training procedure that endows diffusion-based Boltzmann samplers with high-ESS, trajectory-wise importance sampling while incurring negligible inference-time overhead relative to sample generation. Furthermore, we extend VT-DIS to learn both diagonal and full covariance structures, moving beyond isotropic noise.
\item \textbf{Equivariant extension.} We demonstrate how VT-DIS integrates with \(\mathrm{E}(3)\)-equivariant diffusion models, enabling the unbiased generation of molecular samples that respect desired symmetry constraints.
\item \textbf{Empirical validation.} On Gaussian mixtures, toy many-particle systems, and molecular benchmarks (DW-4, LJ-13, and alanine dipeptide), VT-DIS significantly increases ESS relative to naive diffusion + IS with the same number of sampling steps, while incurring only a fraction of the computational cost required for exact likelihood computation via the probability-flow ODE.
\end{itemize}

\section{Related Work}
\label{sec:related}

Normalizing-flow (NF) Boltzmann generators learn an invertible mapping from a simple base distribution to the Boltzmann target, enabling exact evaluation of sample densities and unbiased importance sampling via the change-of-variables formula~\citep{noe2019boltzmann}. Subsequent work improved expressiveness and scalability: Continuous Normalizing Flows (CNFs) based on neural ODEs~\citep{chen2018neural,grathwohl2018ffjord}; symmetry-aware flows introduce \(\mathrm{E}(3)\)-equivariance for molecular systems~\citep{klein2023equivariant}. All CNF variants, however, still require computationally expensive Jacobian determinants or ODE integration for exact likelihood evaluation, making unbiased sampling prohibitive for large biomolecules.

Score-based diffusion models learn the noise-perturbed score and generate samples either by reversing the stochastic differential equation (SDE)~\citep{sohl2015deep,ho2020denoising} or by integrating the probability-flow ODE, which yields exact likelihoods~\citep{song2019generative}. Although diffusion methods have proven effective for molecular data~\citep{hoogeboom2022equivariant}, evaluating proposal log densities via the PF ODE still requires Jacobian-trace computation and becomes expensive in high dimensions. 


Beyond the two popular approaches considered above, neural samplers \citep{vargas2023denoising,richter2023improved,vargas2023transport,zhang2021path} can, in principle, generate unbiased samples directly from a target energy function without requiring training data. However, \citet{he2025no} show that these methods typically incur substantial computational overhead and are often less efficient than generating training data via MCMC and subsequently fitting a diffusion model. Scalability is also a concern: to the best of our knowledge, no neural sampler currently performs well on molecules in 3D Cartesian coordinates—the focus of this work. We therefore omit neural-sampler baselines from our experiments.

\textbf{VT-DIS: Bridging the Gap.}  
VT-DIS retains the efficient stochastic reverse trajectory of standard diffusion models while enabling unbiased reweighting. Instead of integrating the ODE, VT-DIS learns per-step Gaussian covariances by minimizing the \(\alpha\)-divergence (\(\alpha = 2\)) between the forward diffusion and reverse sampling processes, thereby maximizing ESS. Concurrent work has tuned covariances under the Kullback–Leibler divergence to improve sample quality~\citep{ou2024improving,bao2022estimating}; VT-DIS differs by explicitly optimizing the \(\alpha\)-divergence objective for importance-sampling efficiency.

\section{Background}
Before presenting our method, we review three key ingredients: (i) score‐based diffusion models (DMs), (ii) their $\mathrm{E}(3)$‐equivariant extensions, and (iii) the role of importance sampling (IS) in obtaining unbiased estimators at inference time.

\subsection{Score-Based Diffusion Models}
\label{section:score_based_diffusion_models}
Score-based diffusion models (SBDMs)~\citep{sohl2015deep,ho2020denoising,song2020score} generate data by first \emph{diffusing} a sample from the data distribution $\pi_{\text{data}}(\boldsymbol{x})$ with progressively increasing Gaussian noise and then \emph{denoising} it in reverse time.  
The forward process is described by the Itô stochastic differential equation (SDE)
\begin{equation}
\label{eq:forward_sde}
d\boldsymbol{x}_{t}
  \;=\;
  \boldsymbol{f}\!\bigl(\boldsymbol{x}_{t},t\bigr)\,dt
  \;+\;
  g(t)\,d\boldsymbol{w}_{t},
\end{equation}
where $t\in[0,T]$ with $T>0$, $\boldsymbol{f}$ is the drift, $g$ the diffusion coefficient, and $\boldsymbol{w}_{t}$ a standard Wiener process.  
Throughout this work we follow the commonly used \emph{variance-exploding} setting of Karras \textit{et~al.}\citep{karras2022elucidating} and set $\boldsymbol{f}\equiv\boldsymbol{0}$ and $g(t)=\sqrt{2t}$. The forward SDE admits two equivalent reverse-time formulations:
\begin{align}
\label{eq:reverse_sde}
d\boldsymbol{x}_{t} &\,=\,
    -2t\,\nabla_{\!\boldsymbol{x}}\log p_{t}(\boldsymbol{x}_{t})\,dt
    \;+\;
    \sqrt{2t}\,d\bar{\boldsymbol{w}}_{t},
    &&\text{(reverse SDE)}\\
\label{eq:pf_ode}
d\boldsymbol{x}_{t} &\,=\,
    -t\,\nabla_{\!\boldsymbol{x}}\log p_{t}(\boldsymbol{x}_{t})\,dt,
    &&\text{(probability-flow ODE),}
\end{align}
where $p_{t}(\boldsymbol{x}_{t})$ is the marginal density at time~$t$ and $\bar{\boldsymbol{w}}_{t}$ is a Wiener process running backwards from $T$ to $0$.  
In practice we approximate the (intractable) score $\nabla_{\!\boldsymbol{x}}\log p_{t}(\boldsymbol{x}_{t})$ with a neural network $s_{\boldsymbol{\theta}}(\boldsymbol{x}_{t},t)$ trained by score matching~\citep{hyvarinen2005estimation,song2019generative}.

To draw a sample, we integrate either \eqref{eq:reverse_sde} or \eqref{eq:pf_ode} backward from $t=T$ to a small cutoff $\varepsilon\approx0$, and take $\boldsymbol{x}_{\varepsilon}$ as an approximate draw from $\pi_{\text{data}}$.  
Solving the reverse SDE with the discrete DDPM sampler~\citep{ho2020denoising} amounts to the update rule
$p_{\boldsymbol{\theta}}\bigl(\boldsymbol{x}_{n-1}\mid\boldsymbol{x}_{n}\bigr)=
   p\bigl(\boldsymbol{x}_{n-1}\mid
            \boldsymbol{x}_{n},\hat{\boldsymbol{x}}_{0}(\boldsymbol{x}_{n}, t_n;\boldsymbol{\theta})
     \bigr),$
where $\hat{\boldsymbol{x}}_{0}$ is the network’s prediction of the clean signal at step~$n$.  
Given a time grid $\{t_{n}\}_{n=0}^{N}$, we obtain the full trajectory density
$p_{\boldsymbol{\theta}}\bigl(\boldsymbol{x}_{0:N}\bigr)
   =
   p(\boldsymbol{x}_{N})
   \prod_{n=1}^{N}
      p_{\boldsymbol{\theta}}\bigl(\boldsymbol{x}_{n-1}\mid\boldsymbol{x}_{n}\bigr)$.
Both the learned reverse-time transition $p_{\boldsymbol{\theta}}(\boldsymbol{x}_{n-1}\mid\boldsymbol{x}_{n})$ and the reference forward kernel $q(\boldsymbol{x}_{n}\mid\boldsymbol{x}_{n-1})$ are Gaussian; their closed-form expressions are provided in Appendix~\ref{appendix: diffusion with is}.  
Alternatively, one may solve the probability-flow ODE~\eqref{eq:pf_ode} with any off-the-shelf ODE solver to obtain samples at $t=\epsilon$.

\subsection{E(3)–Equivariant Score-Based Diffusion Models}
\label{section:e3_equivariant_dms}

Standard diffusion models ignore the geometric symmetries that characterise many structured datasets.  
In molecular systems, for instance, the potential-energy surface is invariant under global \emph{translations}, \emph{rotations}, and \emph{reflections}—the Euclidean group $\mathrm{E}(3)$—as well as permutations of identical atoms.  
Encoding these symmetries explicitly is critical for producing physically valid and statistically efficient samplers.

\textbf{Equivariant backbone.}
Hoogeboom \textit{et~al.}~\citep{hoogeboom2022equivariant} introduced \emph{$\mathrm{E}(3)$–equivariant diffusion models} by replacing the standard score network with an Equivariant Graph Neural Network (EGNN)~\citep{satorras2021n}.  
Following the analysis of Xu \textit{et~al.}~\citep{xu2022geodiff}, a generative process that is equivariant to $\mathrm{E}(3)$ transformations must satisfy two conditions:

\begin{enumerate}[leftmargin=*]
\item \textbf{Rotation and reflection invariance.}  
      For any orthogonal transformation $R\!\in\!\mathrm{O}(3)$,
      \[
      p(\boldsymbol{x}_{N}) \;=\; p(R\boldsymbol{x}_{N}), 
      \quad
      p(\boldsymbol{x}_{n-1}\mid\boldsymbol{x}_{n})
      \;=\;
      p(R\boldsymbol{x}_{n-1}\mid R\boldsymbol{x}_{n}),
      \]
      which implies that the marginal $p(\boldsymbol{x}_{n})$ is invariant at every time step.

\item \textbf{Translation invariance.}  
      Perfect translation invariance cannot hold for non-degenerate continuous densities; instead, it is imposed by restricting all configurations to the \emph{centre-of-mass (CoM) subspace}
      \(
      \sum_{i}\boldsymbol{x}_{i}=\boldsymbol{0},
      \)
      i.e.\ every sample is translated so that its CoM is at the origin.
\end{enumerate}

\textbf{Practical implementation.}
These constraints are enforced by replacing ordinary Gaussians $\mathcal{N}(\cdot)$ with Gaussians projected onto the zero–CoM subspace, denoted $\mathcal{N}_{\!x}(\cdot)$ (sampling and density evaluation details are given in Appendix~\ref{app:vt-dis-e3}).  
Sampling proceeds exactly as in Sec.~\ref{section:score_based_diffusion_models}: we integrate either the reverse SDE or the probability-flow ODE, but all intermediate states now live in the zero-CoM subspace and all conditional kernels are the projected Gaussians $\mathcal{N}_{\!x}$.  
The resulting pipeline produces molecular configurations that are E(3)-equivariant by construction.

\subsection{Diffusion Models with Importance Sampling}
\label{section:dm_with_is}

Assume that the (unnormalized) target density $\pi_{\text{data}}(\cdot)$ can be evaluated point-wise, whereas drawing i.i.d.\ samples is infeasible.  
Directly using raw outputs from a pretrained diffusion model therefore yields biased estimates of expectations  
$\mathbb{E}_{\pi_{\text{data}}}\bigl[\phi(\boldsymbol{x})\bigr]$.  
IS removes this bias through the weight  
$\pi_{\text{data}}(\boldsymbol{x})/p_{0}(\boldsymbol{x})$,  
where $p_{0}$ denotes the diffusion sampler’s marginal at $t=0$.  
Unfortunately, $p_{0}$ is typically intractable, and any approximation re-introduces bias.  
Because we may generate samples either by integrating the reverse SDE or by solving the probability-flow ODE, two corresponding IS strategies arise.

\textbf{Reverse SDE pathway.}
When sampling with a DDPM-type discretization of the reverse SDE, define
\begin{equation}
\label{eq:joint_distributions}
q(\boldsymbol{x}_{0:N})
 =
 \pi_{\text{data}}(\boldsymbol{x}_{0})
 \prod_{n=0}^{N-1}
     q(\boldsymbol{x}_{n+1}\mid\boldsymbol{x}_{n}),
\qquad
p_{\boldsymbol{\theta}}(\boldsymbol{x}_{0:N})
 =
 p(\boldsymbol{x}_{N})
 \prod_{n=0}^{N-1}
     p_{\boldsymbol{\theta}}(\boldsymbol{x}_{n}\mid\boldsymbol{x}_{n+1}),
\end{equation}
where both kernels are Gaussian and therefore tractable.  
The trajectory weight $w = \frac{q(\boldsymbol{x}_{0:N})}{p_{\boldsymbol{\theta}}(\boldsymbol{x}_{0:N})}$
yields an unbiased estimator (see Appendix~\ref{appendix: diffusion with is}).  
In practice, however, the mismatch introduced by a finite time discretization causes the effective sample size (ESS) to collapse.

\textbf{PF ODE pathway.}
If instead we solve the probability-flow ODE~\eqref{eq:pf_ode}, the likelihood of the terminal sample can be written as~\citep{chen2018neural,song2020score}
\begin{equation}
\label{eq:ode_likelihood}
\log p_{0}\bigl(\boldsymbol{x}(0)\bigr)
 =
 \log p_{T}\bigl(\boldsymbol{x}(T)\bigr)
 +
 \int_{0}^{T}
   \nabla_{\boldsymbol{x}}\cdot\bigl(-t\,s_{\boldsymbol{\theta}}(\boldsymbol{x}(t),t)\bigr)\,dt .
\end{equation}
where $s_{\boldsymbol{\theta}}$ is the score model. Computing the divergence term exactly is prohibitively expensive and is typically approximated with the Hutchinson–Skilling trace estimator~\citep{skilling1989eigenvalues,hutchinson1989stochastic}. Such stochastic estimates, together with time-discretisation error, re-introduce bias and defeat the purpose of IS.

\textbf{Limitations.}
Although both pathways are theoretically sound, each faces practical obstacles: the reverse-SDE approach yields vanishing ESS, whereas the ODE approach incurs excessive compute or bias.  
In the next section we address the first limitation by introducing a covariance-adjusted scheme that dramatically increases ESS in practice.

\section{Method}
\label{section: method}
In this section, we introduce the post-training method VT-DIS to improve ESS for diffusion models with importance sampling. We will first give an overview of the whole algorithm and explain each part of the algorithm in more details and the rationale behind them, then we will explain how our method can also be applied to E(3)-Equivariant diffusion models for molecular generations.

\subsection{Overview}
We build on a pretrained score-based diffusion model, freezing its score network \(s_{\boldsymbol{\theta}}\), and introduce a post-training step that adjusts only the per-step noise covariance \(\{\Sigma_{\boldsymbol{\phi}}(n)\}_{n=1}^{N}\), with $\boldsymbol{\phi}$ being learnable parameters. The objective is to make the learned reverse-time process better agree with the true forward (noising) process by minimizing the \(\alpha=2\) divergence, which directly controls the variance of importance weights and hence the ESS. In practice, we: (1) Simulate target trajectories \(\{\boldsymbol{x}_{0:N}\}\) from the forward SDE; (2) Compute the log–density ratio between the forward and reverse joint distributions given current \(\{\Sigma_{\boldsymbol{\phi}}(n)\}_{n=1}^{N}\); (3) Update the covariance parameters by minimizing the Monte Carlo estimate of the \(\alpha=2\) divergence; (4) At inference, sample with the tuned covariances and assign each trajectory a single importance weight. We summarize the whole process in Algorithm~\ref{alg:vt-dis}.

\begin{algorithm}[t]
  \caption{Post-training Covariance Tuning (VT-DIS)}
  \label{alg:vt-dis}
  \begin{algorithmic}[1]
    \Require Score network \(s_{\boldsymbol{\theta}}\), initial covariances \(\{\Sigma_{\boldsymbol{\phi}}\}\), time grid \(\{t_n\}_{n=0}^N\), batch size \(M\)
    \Ensure  Tuned covariances \(\{\Sigma_{\boldsymbol{\phi}^*}\}\)
    \Repeat
      \State Sample \(\{\boldsymbol{x}_{0:N}^{(i)}\}_{i=1}^M \sim q(\boldsymbol{x}_{0:N})\)
      \State \(\ell^{(i)}(\boldsymbol{\phi}) \gets \log q(\boldsymbol{x}_{0:N}^{(i)}) - \log p_{\{\Sigma_{\boldsymbol{\phi}}\}}(\boldsymbol{x}_{0:N}^{(i)})\)
      \State \(\displaystyle L(\boldsymbol{\phi}) \;\gets\; \frac{1}{M}\sum_{i=1}^M \exp\bigl(\ell^{(i)}(\boldsymbol{\phi})\bigr)\)
      \State \(\{\Sigma_{\boldsymbol{\phi}}\} \!\gets\! \{\Sigma_{\boldsymbol{\phi}}\} - \eta\,\nabla L(\boldsymbol{\phi})\)
    \Until{convergence}
    \State \(\{\Sigma_{\boldsymbol{\phi}^*}\} \gets \{\Sigma_{\boldsymbol{\phi}}\}\)
  \end{algorithmic}
\end{algorithm}

\subsection{Covariance Parameterization}

We explore three complementary ways to parameterize the per-step covariance \(\Sigma_{\boldsymbol{\phi}}(n)\).

\paragraph{Isotropic Covariance}  
The standard DDPM proposal $p_{\boldsymbol{\theta}}(\boldsymbol{x}_{n-1}|\boldsymbol{x}_n)$ uses an isotropic variance \(\Sigma_n = \sigma^2_{\rm ddpm}(n)\boldsymbol{I}\) (where $\sigma^2_{\rm ddpm}(n)$ is defined in Appendix~\ref{appendix: diffusion with is}, \eqref{eq:ddpm-variance}). To tune this, for each $n=1,\cdots, N-1$ we introduce a nonnegative scalar \(\eta_n\) so that
\begin{equation}
\label{equation: isotropic parameterization}
\Sigma_{\boldsymbol{\phi}}(n) = \eta_n\sigma^2_{\rm ddpm}(n)\boldsymbol{I}
\end{equation}
with \(\eta_n\ge0\) enforced (e.g.\ via \texttt{SoftPlus}) and initialized at \(\eta_n=1\) to recover the standard DDPM baseline. We use $\boldsymbol{\phi}$ to represent all tunable parameters $\eta_n$ for each $n$.

\paragraph{Structured covariances}  
To capture per-coordinate anisotropy or correlations, we replace the scalar multiplier with vector or matrix parameters. A diagonal form is
\begin{equation}
\label{equation: diagonal parameterization}
\Sigma_{\boldsymbol{\phi}}(n) = \sigma^2_{\rm ddpm}(n)\,\mathrm{diag}(\eta_{n,1},\dots,\eta_{n,d}),
\end{equation}
with $\eta_{n,i}\geq0$ and initialized with each \(\eta_{n,i}=1\) so that baseline is recovered. A fully correlated covariance uses a learnable matrix \(L_n\in\mathbb{R}^{d\times d}\): $\Sigma_{\boldsymbol{\phi}}(n) = \sigma_{\rm ddpm}^2(n)\,L_n\,L_n^\top$,
with \(L_n\) initialized to the identity to recover the baseline. To limit parameter growth which makes optimization hard especially when number of time steps is large, we further propose a practical low–rank plus isotropic decomposition:
\begin{equation}
\label{equation: low rank parameterization}
\Sigma_{\boldsymbol{\phi}}(n) = \sigma^2_{\rm ddpm}(n)\bigl(A_nA_n^\top + \alpha_n \boldsymbol{I}\bigr),
\end{equation}
where \(A_n\in\mathbb{R}^{d\times k}\) where $k<d$ is chosen depending on the computing budget and \(\alpha_n\ge0\) are learned, and \(A_nA_n^\top+\alpha_n \boldsymbol{I}\) is initialized to \(I\) by initializing $A_n$ to be zero matrix and $\alpha_n=1$.

\paragraph{Sample-dependent models} 
In the case that the number of time steps is large and even the low-rank version (\eqref{equation: low rank parameterization}) becomes infeasible, we can use a neural network with input $t_n$ to predict the optimal covariance matrix at each step. Additionally, for maximal flexibility, \(\Sigma_{\boldsymbol{\phi}}(n)\) can be made a function of the current state \(\boldsymbol{x}_n\) (and \(t_n\)) via a neural network. That is:
\begin{equation*}
\Sigma_{\boldsymbol{\phi}}(n) = \text{NN}_{\boldsymbol{\phi}}(t_n)\quad\text{or}\quad\Sigma_{\boldsymbol{\phi}}(n) = \text{NN}_{\boldsymbol{\phi}}(\boldsymbol{x}_n, t_n)
\end{equation*}
Although this state-dependent form remains valid, our experiments did not show significant ESS gains over the time–only or time-and-sample dependent parameterizations and the additional function evaluation to get covariance increases the computational cost during inference of diffusion models. Therefore, we defer its exploration to future work.

\subsection{Extension to E(3)-Equivariant Diffusion}
\label{section: e3 diffusion with importance sampling}
Our VT-DIS covariance tuning readily extends to E(3)-equivariant diffusion models: the only modification is that all trajectories lie in the zero–center–of–mass (CoM) subspace. Isotropic covariances require no change, since it commutes with rotations, reflections, and CoM projection.

For diagonal or fully correlated covariances, the following form of covariance matrix will satisfy E(3)-equivariance and permutation equivariance:
\begin{equation}
\label{equation: e3 cov parameterization}
\Sigma_{\boldsymbol{\phi}}(n) \;=\; B_{\boldsymbol{\phi}}(n) \,\otimes\, I_3,
\end{equation}
where \(B_{\boldsymbol{\phi}}(n)\in\mathbb{R}^{M\times M}\) acts on the \(M\) particles, \(I_3\) on the three Cartesian coordinates and $\otimes$ denotes the Kronecker product. The matrices \(B_{\boldsymbol{\phi}}(n)\) are learned per step. This form automatically satisfies E(3)-equivariance (proof in Appendix~\ref{app:vt-dis-e3}). To enforce permutation symmetry, \(B_{\boldsymbol{\phi}}(n)\) must respect the specific particle-identities of the dataset. We illustrate two representative cases below.

\textbf{Many–particle systems} 
All particles are identical, so a diagonal \(B_{\boldsymbol{\phi}}(n)\) reduces to an isotropic covariance and adds no flexibility. Instead, we parameterize the full covariance
\begin{equation*}
B_{\boldsymbol{\phi}}(n)=(b_n-a_n)\boldsymbol{I} + a_n\boldsymbol{1}\boldsymbol{1}^\top
\end{equation*}
where $a_n$ and $b_n$ are learnable parameters with \(b_n>0\) and \(-\tfrac{b_n}{M-1} < a_n < b_n\) to ensure \(B_{\boldsymbol{\phi}}(n)\succ 0\). Appendix~\ref{app:vt-dis-e3} shows that this choice satisfies both E(3)-equivariance and full permutation invariance, as needed for systems such as $n$-body systems.

\paragraph{Molecular distributions.}
Real‐world molecules introduce chemical symmetries that are more subtle than the
“all–particles–indistinguishable’’ setting of the above section.
For instance, in alanine dipeptide only the three hydrogens bound to the same
carbon atom may be permuted freely; the remaining atoms are fixed.  To encode
such partial symmetries we assign every atom a symmetry label
\(
    L_i \in \{1,\dots,K\},
    \;i=1,\dots,M,
\)
where atoms sharing the same label are interchangeable while atoms with
different labels are not.  The covariance
\(
    B_{\boldsymbol{\phi}}(n)
\)
at sampling step $n$ is then constrained to depend on the pair $(i,j)$
\emph{only} through the labels $(L_i,L_j)$. Proof in Appendix~\ref{appendix: permutation equivariance for general molecule distributions}.

\begin{enumerate}[leftmargin=*]
\item \textbf{Diagonal (per–class variance).}\,
      Let
      \(
          \eta_{n,\ell}\ge 0
      \)
      be a learnable variance for class~$\ell$.
      \[
          B_{\boldsymbol{\phi}}(n)
          \;=\;
          \mathrm{diag}\bigl(
              \eta_{n,L_1},
              \eta_{n,L_2},
              \dots,
              \eta_{n,L_M}
          \bigr).
      \]
      Positive definiteness is automatic when every $\eta_{n,\ell}>0$.

\item \textbf{Block‐structured (full class–class covariance).}\,
      For any learnable matrix
      \(A_n\in\mathbb{R}^{K\times K}\) and a learnable scalar
      \(\alpha_n>0\).  Define
      \[
          \bigl[B_{\boldsymbol{\phi}}(n)\bigr]_{ij}
          \;=\;
          (A_nA_n^\top)_{L_iL_j}
          \;+\;
          \alpha_n\,\delta_{ij},
      \]
      where $\delta_{ij}$ is the Kronecker delta.
      The added diagonal term $\alpha_n I$ guarantees
      \(B_{\boldsymbol{\phi}}(n)\succ0\).
\end{enumerate}

\subsection{Objective Function}
Our aim is to align the reverse‐process joint proposal \(p_{\boldsymbol{\theta}}(\boldsymbol{x}_{0:N})\) with the forward‐diffusion joint \(q(\boldsymbol{x}_{0:N})\) (see \eqref{eq:joint_distributions}), in a manner that directly improves importance‐sampling performance.  Although one could minimize the Kullback–Leibler (KL) divergence \(D_{\mathrm{KL}}(q\|p_{\phi})\) (and indeed closed‐form optimal covariances under this criterion have been derived \citep{ou2024improving}), such minimization does not control tail mismatch.  Poor alignment in the tails leads to high variance of importance weights and thus low ESS when sample size is large. Instead, we minimize the \(\alpha\)-divergence with \(\alpha=2\)\footnote{We use the $\alpha$‑divergence
$D_{\alpha}\!\left(q\|p_{\phi}\right)=\frac{1}{\alpha(\alpha-1)}
\!\left(\int q(x)^{\alpha}\,p_{\phi}(x)^{1-\alpha}\,dx - 1\right)$ ($\alpha\neq 0,1$).
In our experiments we minimize the case $\alpha=2$; the optimized loss $\mathcal{L}(\phi)$ is the integral term above.}, which is proportional to the variance of importance weights:
\begin{equation}
\label{eq: objective function}
\arg\min_{\phi} D_{\alpha=2}\bigl(q\|p_{\phi}\bigr)
= \arg\min_{\phi}\left(\mathcal{L}(\phi):= \int \frac{q(\boldsymbol{x}_{0:N})}{p_{\phi}(\boldsymbol{x}_{0:N})}q(\boldsymbol{x}_{0:N})\,\mathrm{d}\boldsymbol{x}_{0:N}\approx\frac{1}{M}\sum_{i=1}^M \frac{q(\boldsymbol{x}^{(i)}_{0:N})}{p_{\phi}(\boldsymbol{x}^{(i)}_{0:N})}\right),\quad\boldsymbol{x}^{(i)}_{0:N}\sim q
\end{equation}

In practice, we draw \(M\) trajectories \(\{\boldsymbol{x}_{0:N}^{(i)}\}\) from the forward process\footnote{Empirically, we found that attempting to estimate \(D_2\) with samples from the learned reverse process \(p_{\theta}\) often induces mode collapse and degrades forward ESS.} \(q\) to form the unbiased Monte Carlo estimate and minimize $\log \mathcal{L}(\phi)$ for numerical stability (see Appendix~\ref{section: implementation of objective}). Optimizing this loss by gradient descent yields covariance parameters that minimize importance weight variance and thereby maximize ESS.

\subsection{Choice of Sampling Schedule}

\citet{karras2022elucidating} propose a family of continuous‐time schedules controlled by a parameter \(\rho\):
\[
t_n \;=\;\Bigl(t_N^{1/\rho} + \tfrac{n}{N-1}\,(t_0^{1/\rho}-t_N^{1/\rho})\Bigr)^{\rho},
\qquad t_0=\epsilon,\;t_N=T.
\]
In the limit \(\rho\to\infty\), this reduces to a fixed geometric grid $t_n = t_0 \Bigl(t_N/t_0\Bigr)^{n/(N-1)}$, which we found consistently yields lower \(\alpha=2\) divergence and higher ESS across all benchmarks. We also experimented with jointly tuning the schedule and per‐step covariances during VT‐DIS. While joint tuning improves performance on simple Gaussian mixtures, it proved unstable on molecular datasets. We defer a robust treatment of schedule learning to future work.

\subsection{Complexity \& Inference‐Time Cost}
Relative to a standard DDPM equipped with importance sampling, \textsc{VT-DIS} adds only a \emph{single, offline} optimization step: tuning the per-step noise covariances \(\{\Sigma_{\boldsymbol{\phi}}(n)\}\).  After this one-time cost, generate samples and importance weights only adds negligible overhead to standard sampling from diffusion models—no extra function evaluations. Unlike flow-based Boltzmann generators or diffusion generate samples by solving PF ODE, \textsc{VT-DIS} never has to integrate an ODE to evaluate marginal proposal log-densities, a procedure that becomes prohibitively expensive in high-dimensional settings.

\section{Experiments}
\label{section: experiments}
\textbf{Experimental setup.} We assume access to both the energy function and the training data used to fit the diffusion model (same assumption as in, e.g., \citet{klein2023equivariant,tan2025scalable}). We first train a diffusion model with an EGNN backbone \citep{satorras2021n} and EDM preconditioning \citep{karras2022elucidating} using a standard training procedure. We then apply VT-DIS to the same training set to optimize the proposal covariances. For evaluation, we compare VT-DIS with a DDPM+IS baseline under a fixed compute budget measured by the number of function evaluations (NFEs; i.e., sampling steps). We also compare against importance sampling with exact likelihoods computed via the probability-flow ODE (PF-ODE). Performance is reported as both forward and reverse effective sample size (ESS). We additionally report the Evidence Lower Bound (ELBO) and Evidence Upper Bound (EUBO) for the DDPM baseline and VT-DIS, which serve as complementary likelihood-based metrics to assess the quality of the generated samples.

We consider three representative classes of Boltzmann systems: (i) \textbf{Synthetic}: a high-dimensional two-mode Gaussian mixture (GMM-2); (ii) \textbf{Many-particle}: double-well (DW-4) and Lennard-Jones-13 (LJ-13) clusters~\citep{kohler2020equivariant}; and (iii) \textbf{Small-molecule}: alanine dipeptide (ALDP). For all targets, the (unnormalized) density \(\pi_{\mathrm{data}}\) is available, enabling exact importance weights. We use the atomic coordinate representation for each sample and train our models on these coordinates. Further implementation details are provided in Appendix~\ref{appendix:implementation-details}.

\subsection{\textsc{GMM-2}}\label{subsec:gmm2}
We begin with a two-component isotropic Gaussian mixture in $\mathbb{R}^{d}$:
\[
\pi_{\text{data}}(\boldsymbol{x}) 
= w_1\,\mathcal{N}\!\bigl(\boldsymbol{x};\,\boldsymbol{\mu}_1,\sigma_1^2\boldsymbol{I}\bigr) 
+ w_2\,\mathcal{N}\!\bigl(\boldsymbol{x};\,\boldsymbol{\mu}_2,\sigma_2^2\boldsymbol{I}\bigr),
\]
where $\boldsymbol{\mu}_1 = (1,\dots,1)^\top \in \mathbb{R}^d$, 
$\boldsymbol{\mu}_2 = (-2,\dots,-2)^\top \in \mathbb{R}^d$, 
$\sigma_1^2=\sigma_2^2 = 0.15$, and $w_1 = \tfrac{2}{3},\;w_2 = \tfrac{1}{3}$. We evaluate two ambient dimensions, \(d \in \{50,100\}\).  
Figure~\ref{fig:gmm2_ess} shows forward and reverse ESS as a function of the number of function evaluations (NFE) for \textsc{VT-DIS} with three noise–covariance parameterizations—\emph{isotropic}, \emph{diagonal}, and \emph{full rank}—together with the DDPM\,+\,IS baseline.

The plots reveal three key trends:  
(i) \textsc{VT-DIS} consistently outperforms the baseline in forward ESS, reverse ESS, or both, for every covariance parameterization.  
(ii) The full-rank covariance yields the highest forward and reverse ESS for both values of \(d\).  For \(d = 100\), however, the forward-ESS error bars are wide when $d=100$, indicating residual mismatch between the forward and reverse processes at very high dimension.  
(iii) With isotropic and diagonal covariances, reverse ESS improves as NFE grows, whereas forward ESS stays low, underscoring the importance of modeling cross-dimensional correlations in this setting.

\begin{figure}[t]
    \centering
    \includegraphics[width=0.9\linewidth]{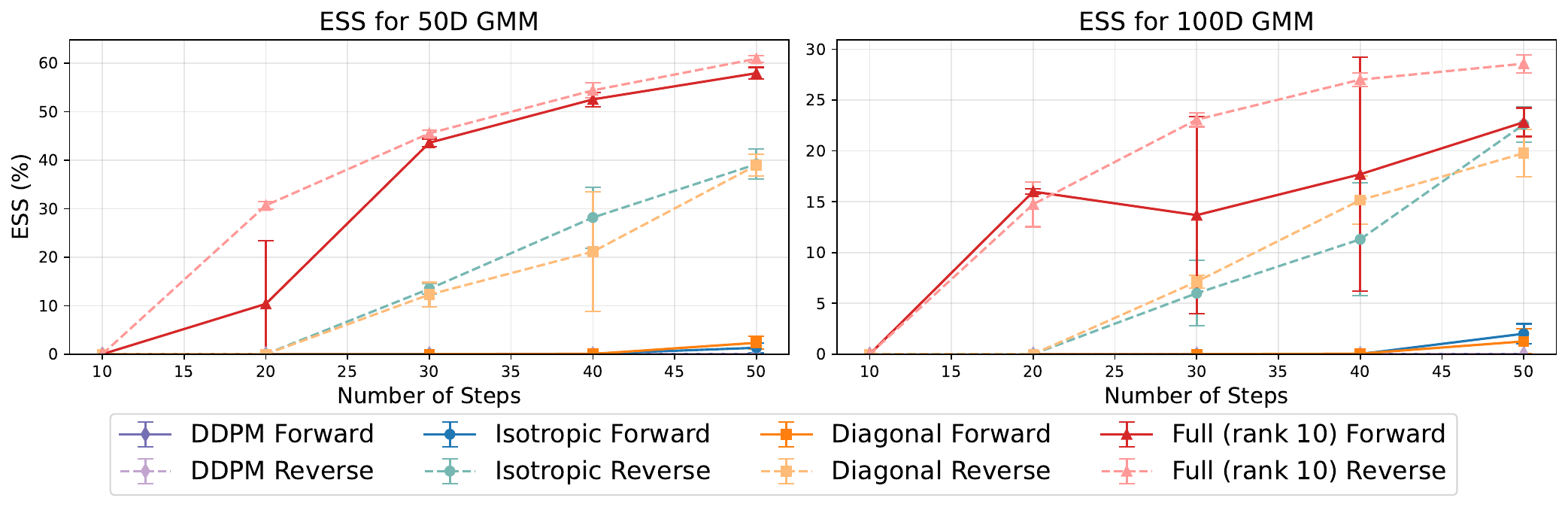}
    \caption{\textbf{GMM-2}: ESS versus NFEs for dimensions \(d=50\) and \(d=100\), comparing VT-DIS with the DDPM + IS baseline. 
    ESS estimates are computed using \(10^5\) samples.}
    \label{fig:gmm2_ess}
\end{figure}

\subsection{Many–particle systems: DW-4 and LJ-13}\label{subsec:dw4_lj13}

We next evaluate VT-DIS on two prototypical many-body systems with pair-wise interactions:
\emph{DW-4} (double-well potential, four particles in \( \mathbb{R}^{2} \), \(d=8\))
and \emph{LJ-13} (Lennard–Jones cluster, thirteen particles in \( \mathbb{R}^{3} \), \(d=39\))\citep{kohler2020equivariant}.

\textbf{ESS versus NFEs.}  
Figure~\ref{fig:ess_dw4_lj13} reports forward and reverse ESS as a function of the number of NFEs. On the low-dimensional DW-4 system, VT-DIS raises both ESS measures from \(\le1\%\) to \(\approx40\%\) using only 100 NFEs. The improvement is even more pronounced on the higher-dimensional LJ-13 cluster: the standard DDPM + IS baseline yields virtually \(0\%\) ESS up to 400 NFEs, whereas VT-DIS achieves \(\approx35\%\) ESS on both forward and reverse measures with the same budget. Due to the permutation-equivariant constraint of the many-particle system, the full covariance matrix has only one additional degree of freedom compared to an isotropic covariance. In this setting, we observe no significant ESS gain from using the full covariance, although both forward and reverse ESS exhibit lower standard deviations than with isotropic noise.  Across both tasks, ESS increases monotonically with the number of NFEs, facilitating a clear accuracy–cost trade-off.
\begin{figure}[t]
  \centering
  \begin{subfigure}[b]{0.32\linewidth}
    \centering
    \includegraphics[width=\linewidth]{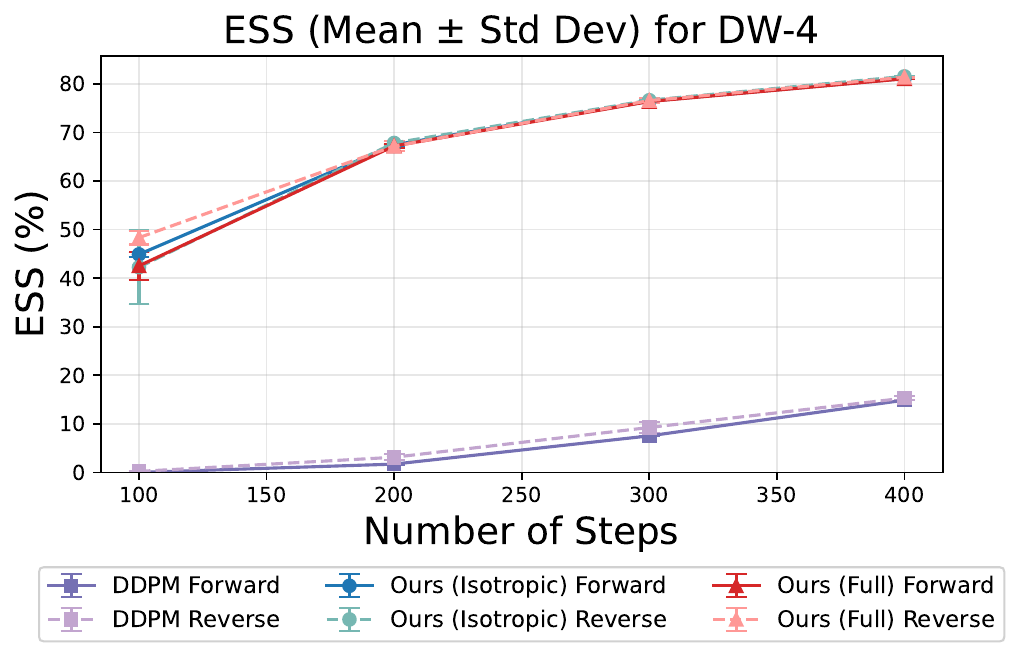}
    \caption{DW-4 (\(d=8\))}
    \label{fig:ess_dw4}
  \end{subfigure}
  \begin{subfigure}[b]{0.32\linewidth}
    \centering
    \includegraphics[width=\linewidth]{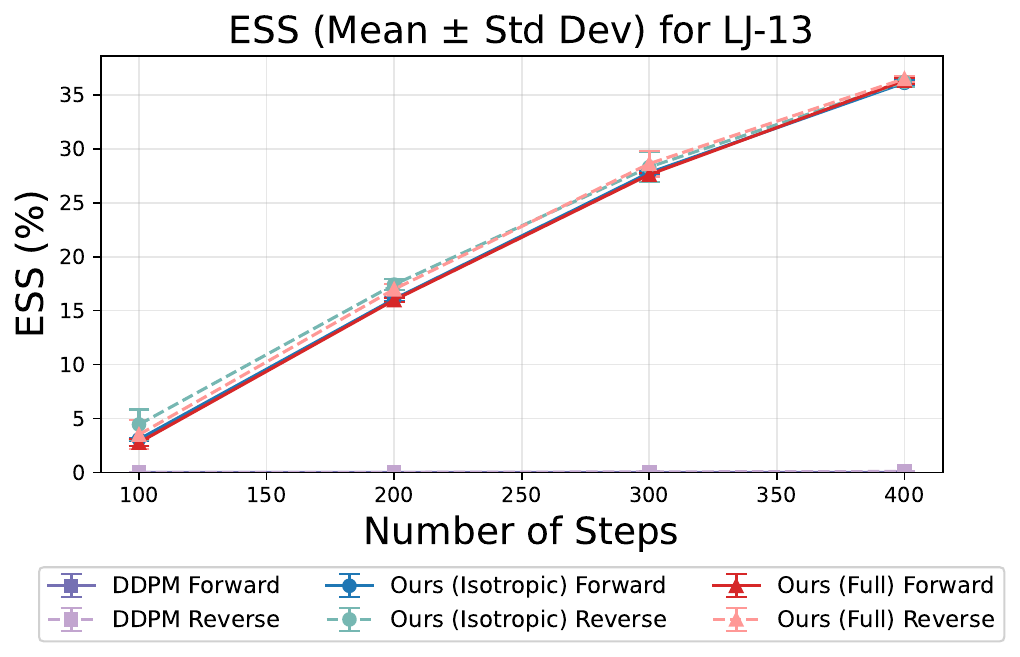}
    \caption{LJ-13 (\(d=39\))}
    \label{fig:ess_lj13}
  \end{subfigure}
  \begin{subfigure}[b]{0.32\linewidth}
    \centering
    \includegraphics[width=\linewidth]{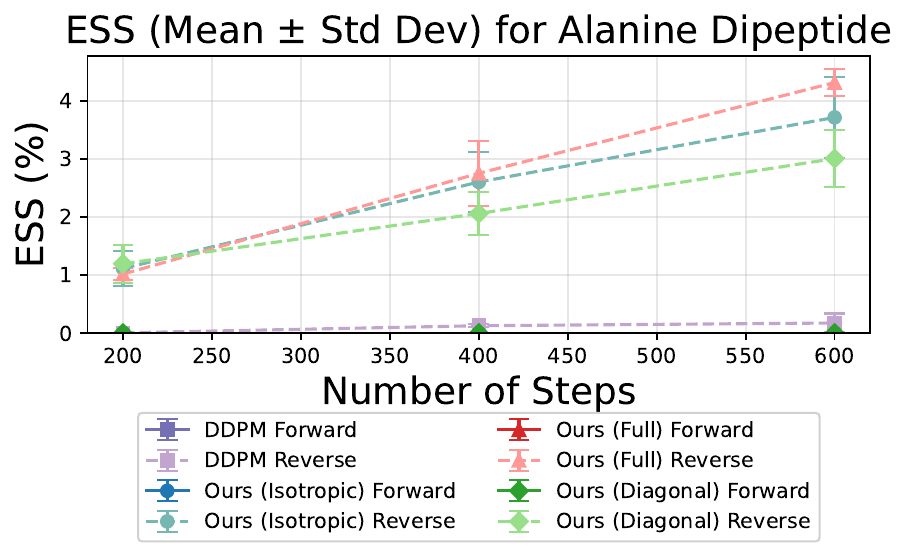}
    \caption{ALDP ($d=66$)}
    \label{fig:ala2_ess}
  \end{subfigure}
\caption{\textbf{ESS versus NFEs.} VT-DIS consistently outperforms the standard DDPM + IS baseline across all tasks under the same computational budget. ESS values are estimated using \(10^5\) samples.}
  \label{fig:ess_dw4_lj13}
\end{figure} 

\textbf{ELBO and EUBO.} We also report the standard likelihood bounds ($\mathrm{ELBO}\le\log p_\theta(\boldsymbol{x})\le\mathrm{EUBO}$), where a smaller gap indicates better agreement between the forward noising path and the learned reverse process. Figure~\ref{fig:elbo_dw4_lj13_aldp} shows that across DW-4 and LJ-13, \textsc{VT-DIS} consistently raises ELBO and lowers EUBO relative to the DDPM baseline, yielding a markedly tighter “sandwich” at every step budget. Overall, these trends mirror the ESS gains, indicating better forward–reverse matching under fixed compute.

\begin{figure}[t]
  \centering
  \begin{subfigure}[b]{0.32\linewidth}
    \centering
    \includegraphics[width=\linewidth]{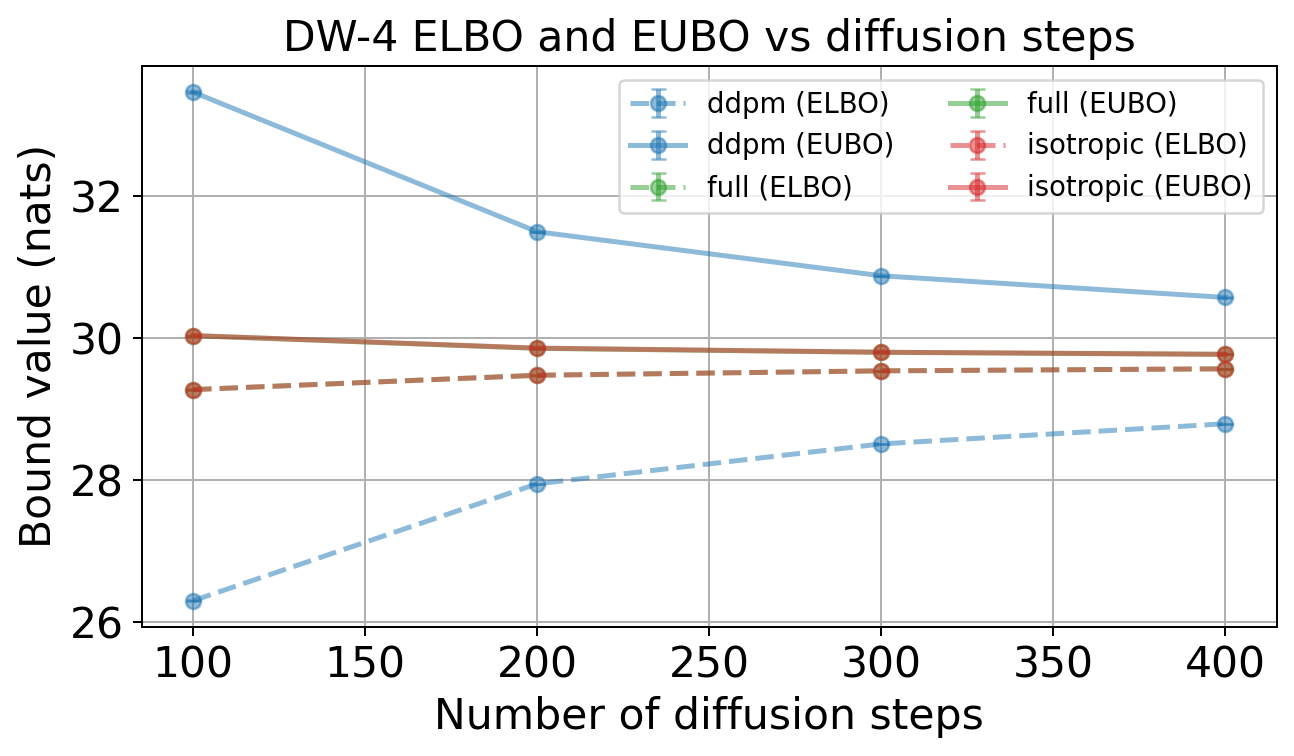}
    \caption{DW-4 (\(d=8\))}
    \label{fig:dw4_elbo}
  \end{subfigure}
  \begin{subfigure}[b]{0.32\linewidth}
    \centering
    \includegraphics[width=\linewidth]{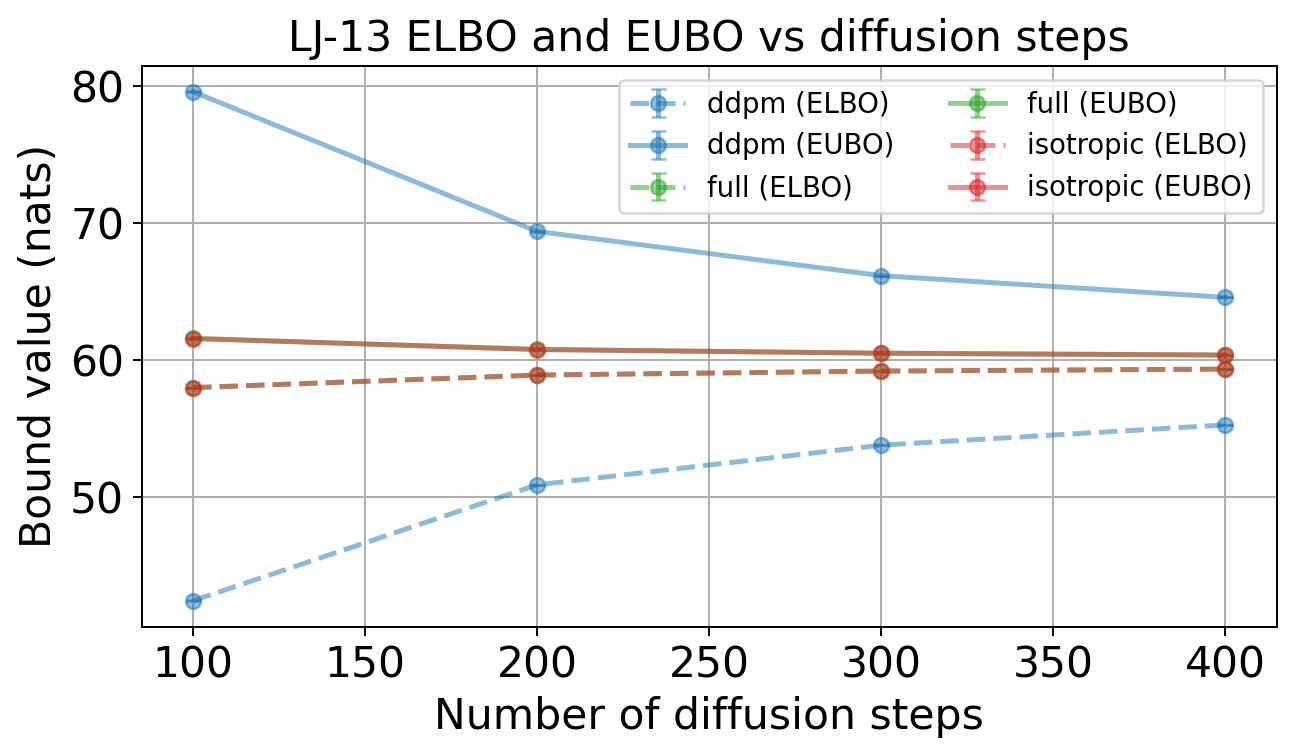}
    \caption{LJ-13 (\(d=39\))}
    \label{fig:lj13_elbo}
  \end{subfigure}
  \begin{subfigure}[b]{0.32\linewidth}
    \centering
    \includegraphics[width=\linewidth]{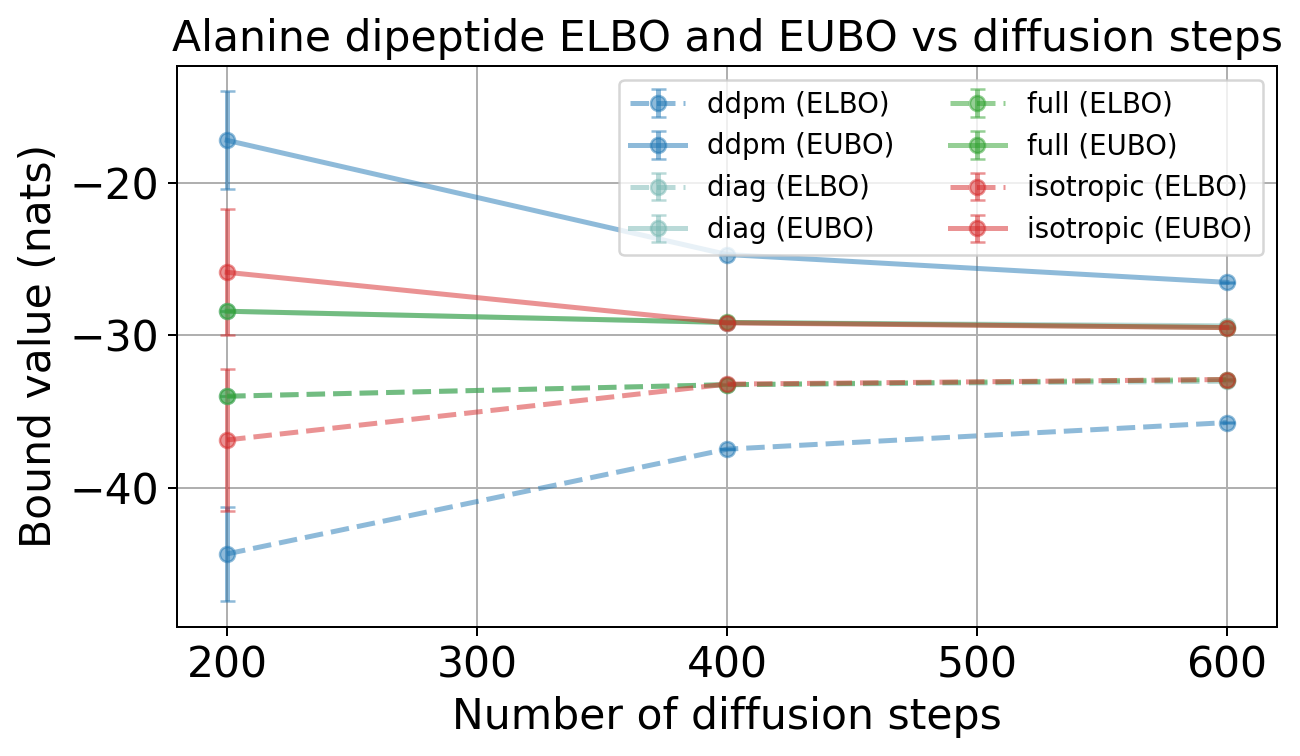}
    \caption{ALDP ($d=66$)}
    \label{fig:ala2_elbo}
  \end{subfigure}
\caption{\textbf{ELBO/EUBO vs. diffusion steps.} ELBO (higher is better) and EUBO (lower is better) for the DDPM baseline and \textsc{VT-DIS} (isotropic, diagonal and full covariances) as a function of the number of sampling steps. Error bars denote $\pm$ one standard deviation across 3 repeated MC evaluations.}

  \label{fig:elbo_dw4_lj13_aldp}
\end{figure}

\textbf{Bias correction.}
To illustrate bias correction we compare three energy histograms for each task:
(i) \emph{ground-truth} samples from the test dataset,
(ii) \emph{unweighted} VT-DIS generated samples,
and (iii) \emph{reweighted} samples.
Figure~\ref{fig:hist_dw4_lj13} shows that
the raw proposals are biased,
but importance reweighting realigns the distributions with the target,
corroborating the ESS improvements.

\begin{figure}[t]
  \centering
  \begin{subfigure}[b]{0.32\linewidth}
    \centering
    \includegraphics[width=\linewidth]{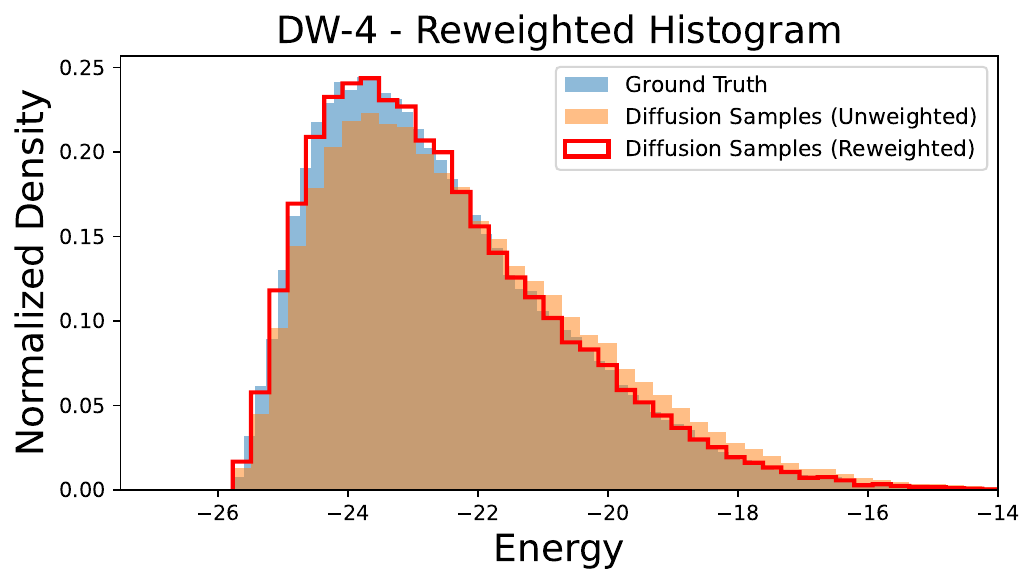}
    \caption{DW-4, 100 NFEs}
    \label{fig:hist_dw4}
  \end{subfigure}
  \begin{subfigure}[b]{0.32\linewidth}
    \centering
    \includegraphics[width=\linewidth]{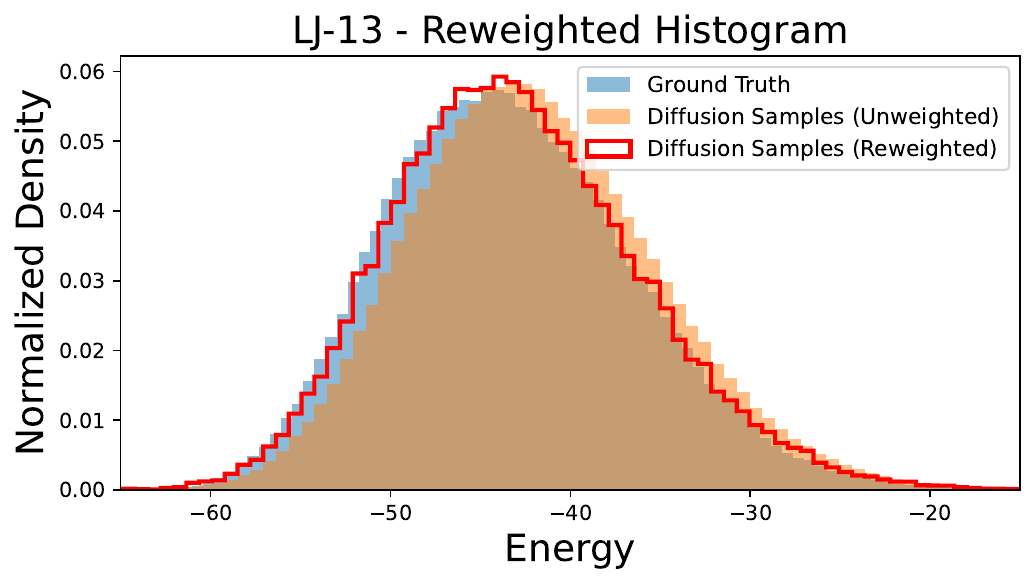}
    \caption{LJ-13, 100 NFEs}
    \label{fig:hist_lj13}
  \end{subfigure}
  \begin{subfigure}[b]{0.32\linewidth}
    \centering
    \includegraphics[width=\linewidth]{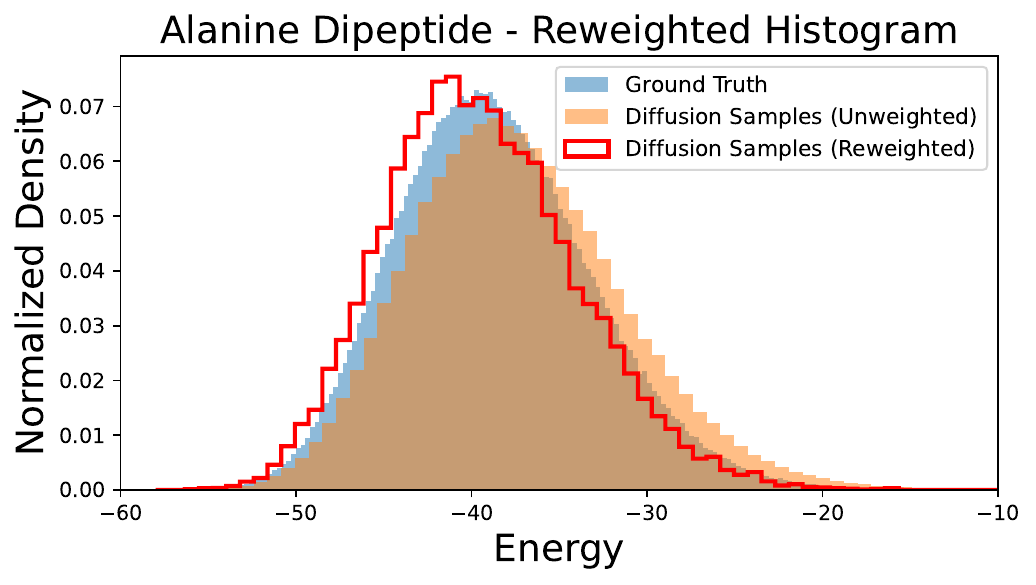}
    \caption{ALDP, 600 NFEs}
    \label{fig:ala2_hist}
  \end{subfigure}
\caption{\textbf{Energy distributions.} Unweighted samples generated by VT-DIS with isotropic covariance (orange) are biased relative to ground-truth samples (blue), but importance-weighted samples (red) effectively correct the discrepancy. Histograms are plotted using \(5\times10^5\) samples.}
  \label{fig:hist_dw4_lj13}
\end{figure}

\subsection{Small Molecule: Alanine Dipeptide (ALDP)}
\label{subsec:ala2}

We evaluate VT-DIS on the 66-dimensional Boltzmann distribution of alanine dipeptide (ALDP) in implicit solvent at temperature \(T = 800\,\mathrm{K}\). The unnormalized Boltzmann density is defined by a classical force-field potential, and the availability of its exact log density enables unbiased importance-sampling evaluation. Training and implementation details are provided in Appendix~\ref{appendix:implementation-details}.

\textbf{Results.} Figure~\ref{fig:ala2_ess} shows the forward and reverse ESS over a range of 200–600 NFEs. VT-DIS consistently improves the reverse ESS, reaching approximately 2.5\% at 400 NFEs, compared to under 0.5\% for the baseline. Among the covariance parameterizations, the full covariance version achieves slightly higher reverse ESS compared to diagonal and isotropic parameterizations. Figure~\ref{fig:ala2_elbo} further illustrates that VT-DIS with all three parameterizations yields a smaller gap between ELBO and EUBO relative to the DDPM baseline. Notably, the full covariance parameterization exhibits an even smaller ELBO–EUBO gap in the low diffusion-step regime, which aligns with the observed ESS performance.

Figure~\ref{fig:ala2_hist} displays the energy histogram, and Figure~\ref{fig: ram plot aldp} presents the Ramachandran plot. Both demonstrate that importance sampling effectively corrects bias in the generated samples. Notably, the molecular dynamics (MD) “ground-truth” samples themselves exhibit a slight bias, as the reweighted energy histogram shifts toward lower energies than the MD samples.

For both VT-DIS and the baseline, the forward ESS is nearly zero due to a small number of extreme density-ratio outliers. This indicates that both methods' joint proposal distributions assign extremely low probability to regions containing true samples from the test dataset. We hypothesize that the score model is missing certain modes, supported by three empirical observations: (i) in the unreweighted samples (middle plot of Figure~\ref{fig: ram plot aldp}), the mode near $\phi=\pi/2$ is underrepresented; (ii) running the DDPM baseline with very large numbers of time steps (1000 and 2000) still resulted in forward ESS $\approx 0$, indicating a support-coverage mismatch rather than discretization issues; (iii) subsequent experiments solving the PF ODE for importance sampling (Table~\ref{tab:dw4_lj13_perf}) similarly yielded forward ESS $\approx 0$ on Alanine Dipeptide.

Since the score model defines the mean of the per-step proposal distributions and VT-DIS only adjusts the covariance, its capacity to improve \emph{forward ESS} is limited if the score model is inadequately trained. Therefore, VT-DIS yields substantial improvement of forward ESS only when the score model sufficiently captures all modes. Two promising future directions include: (i) enhancing the training of the score model to better capture all relevant modes—a direction complementary to our current work, as VT-DIS can enhance any trained diffusion model; and (ii) extending VT-DIS to adjust both the mean and covariance of the per-step proposal distributions, particularly useful when the score model is missing certain modes.

\begin{figure}[t]
    \centering
    \includegraphics[width=1\linewidth]{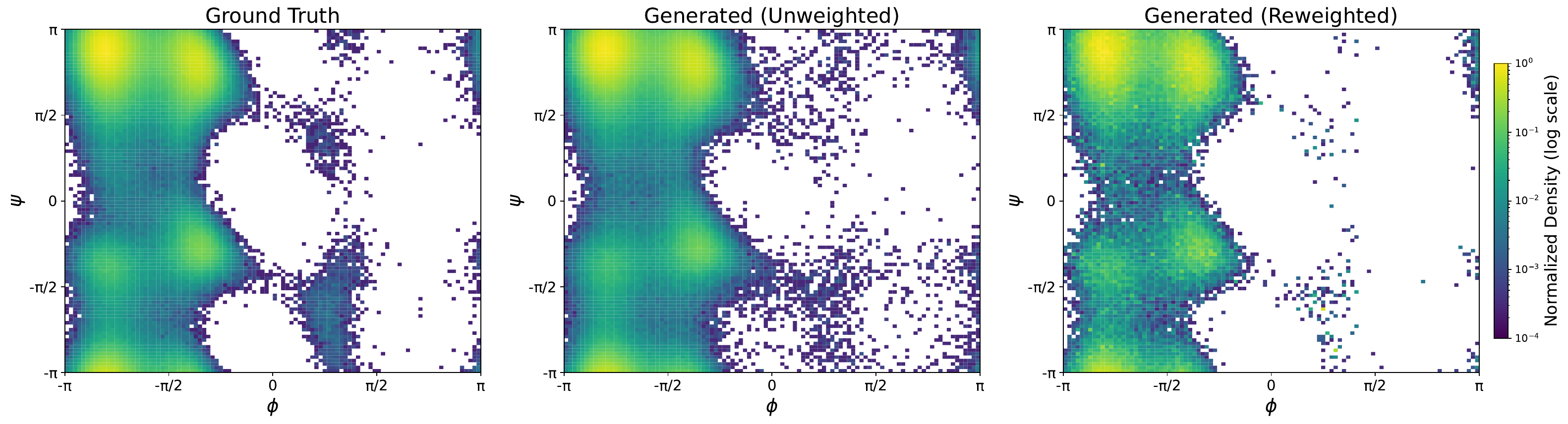}
\caption{Ramachandran plots from left to right: ground-truth samples generated by MD, VT-DIS (isotropic) samples before and after reweighting. All plots are generated using \(10^6\) samples.}
    \label{fig: ram plot aldp}
\end{figure}

\subsection{Computational Cost versus Probability–Flow ODE}
\label{subsubsec:cnf_comparison}
While VT-DIS and PF-ODE + IS both yield unbiased estimators, they differ substantially in computational efficiency. Here, we compare the ESS and computational cost of each method.

\textbf{Experimental setup.}  
To isolate inference cost, we reuse the \emph{same} pretrained score network and employ an identical number of integration steps for both methods. Computational effort is measured in \emph{giga floating-point operations (GFLOPs) per sample}, and accuracy is quantified by forward and reverse ESS. Results for DW-4, LJ-13, and ALDP are summarized in Table~\ref{tab:dw4_lj13_perf}. We also include an extended discussion and results on solving the PF-ODE using the trace estimator in Appendix~\ref{section: extended computational cost}.

\textbf{Results.}  
PF-ODE + IS attains higher ESS but incurs a steep computational penalty: its GFLOPs per sample are one to two orders of magnitude greater than VT-DIS’s, and this gap widens with dimensionality. In practice, PF-ODE requires a \emph{minimum} number of ODE steps to control proposal-density bias; each additional step further increases cost. VT-DIS, in contrast, trades some ESS for dramatically lower compute. Its estimator is asymptotically unbiased as the sample number increases, regardless of the time-step budget, and each additional reverse-SDE step directly boosts ESS.

\begin{table}[t]
  \centering
  \caption{Performance on DW-4, LJ-13 and ALDP with $10^{4}$ samples on a single A100. The number in parentheses after each dataset denotes the sampling steps. $\uparrow$ = higher is better, $\downarrow$ = lower is better.}
  \label{tab:dw4_lj13_perf}
  \begin{tabular}{@{}lcccccc@{}}
    \toprule
    & \multicolumn{2}{c}{\textbf{DW-4} (100)} & \multicolumn{2}{c}{\textbf{LJ-13} (200)} & \multicolumn{2}{c}{\textbf{ALDP} (200)}\\
    \cmidrule(lr){2-3}\cmidrule(l){4-5}\cmidrule(l){6-7}
    & {Rev. SDE} & {PF ODE} & {Rev. SDE} & {PF ODE} & {Rev. SDE} & {PF ODE}\\
    \midrule
    Forward ESS (\%) $\uparrow$  &  44.32 & 94.29 &  15.87 &    79.21  & $\approx 0$ & $\approx 0$ \\
    Reverse ESS (\%) $\uparrow$     &  46.76 & 83.60 & 17.69 &    77.42   & 1.05 & 35.39 \\
    GFlops / sample $\downarrow$       & 0.81      &  
                 7.70           & 26.2      &   2173     & 33.7 & 45967\\
    \bottomrule
  \end{tabular}
\end{table}

\section{Conclusion and Limitations}
\label{section: conclusion}

\textbf{Contributions.} We have introduced VT-DIS, a post-training procedure that can be applied to any pretrained diffusion model, and demonstrated its effectiveness in molecular generation. By computing trajectory-wise importance weights, VT-DIS achieves high effective sample size (ESS) across multiple benchmarks while using only a fraction of the computational budget required to solve the PF ODE for exact likelihood estimation. Furthermore, VT-DIS provides asymptotically unbiased expectation estimates as the number of samples grows—independent of the number of sampling steps. In contrast, PF-ODE solvers guarantee asymptotically unbiased sampling only when both the number of samples and the number of discretization steps increase, which limits their practical utility. These findings highlight VT-DIS’s potential for diffusion-based Boltzmann generators.

\textbf{Limitations and Future Work.} Despite its advantages, VT-DIS has several limitations. First, the post-training procedure must be repeated for each sampling schedule, which can incur significant computational and memory costs. Figure~\ref{fig:tuned_parameters} in Appendix~\ref{appendix: add experiments optimized cov} shows a clear pattern in optimal hyperparameters across different step counts; future work could leverage this pattern to develop a unified post-training algorithm that supports arbitrary schedules in one stage. Second, optimizing the sampling schedule itself to minimize the \(\alpha\)-divergence remains an open question; existing methods from \citep{sabour2024align,williams2024score} may offer a promising starting point. Third, VT‑DIS requires access to ground‑truth samples to tune the per‑step proposal covariances. When such samples are unavailable, a dataset must first be generated—e.g., via MCMC or molecular dynamics (MD). A useful direction for future work is to adapt VT‑DIS to operate with only a small set of approximate reference samples, which are easier to obtain.

\section*{Acknowledgments}
We thank Jiajun He and Javier Antorán for valuable discussions during the development of this work, and Ruikang Ouyang and Zhuoyue Huang for their insightful feedback on earlier drafts.  
FZ and JMHL acknowledge support from the Turing AI Fellowship under UKRI grant EP/V023756/1.  
LIM acknowledges support from Google’s TPU Research Cloud program and the EPSRC through the SynTech CDT.

\bibliography{main}

\newpage

\appendix
\section{Score-based Diffusion Models with Importance Sampling}

In this appendix, we present the mathematical derivations for combining score-based diffusion models with importance sampling. We first review the fundamentals of importance sampling and then show how it can be integrated into score-based diffusion to correct bias in generated samples.

\subsection{Review of Importance Sampling}

Importance sampling (IS) provides an estimator for expectations of the form
\[
\mathbb{E}_{\boldsymbol{x}\sim \pi_{\mathrm{data}}}[\phi(\boldsymbol{x})]
= \int \phi(\boldsymbol{x})\,\pi_{\mathrm{data}}(\boldsymbol{x})\,\mathrm{d}\boldsymbol{x},
\]
where $\pi_{\mathrm{data}}$ is the (possibly unnormalized) target density and $\phi$ is a test function. Directly sampling from $\pi_{\mathrm{data}}$ may be infeasible, but we assume that we can evaluate its density up to a normalization constant.  Instead, we draw samples $\{\boldsymbol{x}^{(n)}\}_{n=1}^N$ from a proposal distribution $p(\boldsymbol{x})$ and form the self-normalized estimator
\[
\mathbb{E}_{\boldsymbol{x}\sim \pi_{\mathrm{data}}}[\phi(\boldsymbol{x})]
\approx
\sum_{n=1}^N \bar{w}_n \,\phi\bigl(\boldsymbol{x}^{(n)}\bigr),
\]
where
\[
w_n = \frac{\pi_{\mathrm{data}}\bigl(\boldsymbol{x}^{(n)}\bigr)}{p\bigl(\boldsymbol{x}^{(n)}\bigr)},
\qquad
\bar{w}_n = \frac{w_n}{\sum_{m=1}^N w_m}.
\]
This self-normalized importance sampling (SNIS) estimator is consistent and asymptotically unbiased, with its variance decreasing as $N$ increases.

\subsection{Effective Sample Size}

The quality of an IS estimator is often measured by its effective sample size (ESS). Two common definitions are:
\begin{itemize}[left=0in]
  \item \emph{Forward ESS}, which evaluates the overlap with the target distribution:
  \begin{equation}
    \label{eq:ess-forward}
    \widehat{\mathrm{ESS}}_{\mathrm{fwd}}
    = \frac{N^2}{Z^{-1} \sum_{i=1}^N w\bigl(\boldsymbol{x}^{(i)}\bigr)}
    ,\quad \boldsymbol{x}^{(i)}\sim \pi_{\mathrm{data}},
  \end{equation}
  where $Z^{-1}\approx\frac{1}{N}\sum_{n=1}^N \frac{1}{w_n}$ is the inverse normalizing constant.
  \item \emph{Reverse ESS}, which measures the concentration of weights under the proposal:
  \begin{equation}
    \label{eq:ess-reverse}
    \widehat{\mathrm{ESS}}_{\mathrm{rev}}
    = \frac{\bigl(\sum_{i=1}^N w(\boldsymbol{y}^{(i)})\bigr)^2}
           {\sum_{i=1}^N w(\boldsymbol{y}^{(i)})^2}
    ,\quad \boldsymbol{y}^{(i)}\sim p.
  \end{equation}
\end{itemize}
While reverse ESS is widely used for evaluation, it can overestimate performance if $p$ suffers from mode collapse.  Forward ESS, by contrast, tends to provide a more reliable measure of overlap between the target and proposal.

\subsection{Combining Score-based Diffusion Models with Importance Sampling}
\label{appendix: diffusion with is}
We now consider task of estimating the integral $\mathbb{E}_{\boldsymbol{x}\sim\pi_{\text{data}}}[\phi(\boldsymbol{x})]$ where normalized or unnormalized $\pi_{\text{data}}$ is the target density which we can evaluate, and $\phi$ is the test function of interest. We aim to train a diffusion models to act as the proposal distribution which we can draw samples from. Suppose we have a time interval $[\epsilon, T]$ and we discretize the time interval into $N$ sections with $N+1$ time steps such that $\epsilon=t_0<t_1<\dots<t_N=T$. Assume the conditional proposal distribution of the diffusion model is $p_{\boldsymbol{\theta}}(\boldsymbol{x}_{n-1}|\boldsymbol{x}_n)$, and the conditional noise distribution is $q(\boldsymbol{x}_n|\boldsymbol{x}_{n-1})$ for $n = 1, \dots, N$. Using SNIS, the integral can be estimated as
\begin{align}
\mathbb{E}_{\boldsymbol{x}\sim\pi}[\phi(\boldsymbol{x})] &\approx \mathbb{E}_{\boldsymbol{x}_0\sim\pi}[\phi(\boldsymbol{x}_0)] \\
&= \int \phi(\boldsymbol{x}_0) \pi(\boldsymbol{x}_0) d\boldsymbol{x}_0 \\
&= \int \phi(\boldsymbol{x}_0) \pi(\boldsymbol{x}_0)\left(\prod_{n=1}^N q(\boldsymbol{x}_n|\boldsymbol{x}_{n-1})\right)d\boldsymbol{x}_{0:N} \\
&= \int \phi(\boldsymbol{x}_0) \left(\frac{\pi(\boldsymbol{x}_0) \prod_{n=1}^N q(\boldsymbol{x}_n|\boldsymbol{x}_{n-1})}{p_{\boldsymbol{\theta}}(\boldsymbol{x}_N) \prod_{n=1}^N p_{\boldsymbol{\theta}}(\boldsymbol{x}_{n-1}|\boldsymbol{x}_n)}\right)p_{\boldsymbol{\theta}}(\boldsymbol{x}_{0:N}) d\boldsymbol{x}_{0:N} \\
\label{equation: importance weights ddpm}
&\approx \frac{1}{K} \sum_{k=1}^K\underbrace{\left( \frac{\pi(\boldsymbol{x}^{(k)}_{t_0}) \prod_{n=1}^N q(\boldsymbol{x}^{(k)}_{t_n}|\boldsymbol{x}^{(k)}_{t_{n-1}})}{p_{\boldsymbol{\theta}}(\boldsymbol{x}^{(k)}_N) \prod_{n=1}^N p_{\boldsymbol{\theta}}(\boldsymbol{x}^{(k)}_{t_{n-1}}|\boldsymbol{x}^{(k)}_{t_n})} \right)}_{w_k}\phi(\boldsymbol{x}_0)
\end{align}
where $\boldsymbol{x}^{(k)}_{0:N}$ for $k = 1, \dots, K$ are samples from the joint proposal distribution $p_{\boldsymbol{\theta}}$ and $w_k$ for $k=1,\dots, K$ are the importance weights. Note that $\epsilon$ is chosen such that the error for the first approximation is negligible. 

We now need to specify the exact forms of $q(\boldsymbol{x}_n|\boldsymbol{x}_{n-1})$ and $p_{\boldsymbol{\theta}}(\boldsymbol{x}_{n-1}|\boldsymbol{x}_n)$. Recall the forward SDE with zero drift ($\boldsymbol{f}\equiv \boldsymbol{0}$) and diffusion coefficient $g(t)=\sqrt{2t}$:
\[
\mathrm{d}\boldsymbol{x}_t = \sqrt{2t}\,\mathrm{d}\boldsymbol{w}_t.
\]
Given a time discretization \(\epsilon = t_0 < t_1 < \cdots < t_N = T\), the conditional \emph{target} distribution of the forward process is Gaussian:
\begin{equation}
\label{eq:ddpm-forward-target}
p(\boldsymbol{x}_n \mid \boldsymbol{x}_{n-1})
= \mathcal{N}\bigl(\boldsymbol{x}_n;\,\boldsymbol{x}_{n-1},\,(t_n^2 - t_{n-1}^2)\,\boldsymbol{I}\bigr),
\qquad t_n > t_{n-1}.
\end{equation}

By Bayes’ rule, the DDPM-style \emph{reverse} (proposal) distribution conditional on \(\boldsymbol{x}_n\) and \(\boldsymbol{x}_0\) is
\begin{align}
\label{eq:ddpm-reverse-proposal}
p(\boldsymbol{x}_{n-1}\mid \boldsymbol{x}_n, \boldsymbol{x}_0)
&\;\propto\;
p(\boldsymbol{x}_n \mid \boldsymbol{x}_{n-1}) \;p(\boldsymbol{x}_{n-1}\mid \boldsymbol{x}_0)\\[-0.5ex]
&=\;
\mathcal{N}\!\bigl(\boldsymbol{x}_{n-1};\,\boldsymbol{\mu}(\boldsymbol{x}_n,\boldsymbol{x}_0),\,\sigma_{\mathrm{ddpm}}^2(n)\,\boldsymbol{I}\bigr),
\nonumber
\end{align}
where
\begin{align}
\boldsymbol{\mu}(\boldsymbol{x}_n,\boldsymbol{x}_0)
&= \frac{t_n^2}{t_{n-1}^2}\,\boldsymbol{x}_n \;+\;\Bigl(1-\tfrac{t_n^2}{t_{n-1}^2}\Bigr)\,\boldsymbol{x}_0,\\
\label{eq:ddpm-variance}
\sigma_{\mathrm{ddpm}}^2(n)
&= \frac{t_{n-1}^2\,\bigl(t_n^2 - t_{n-1}^2\bigr)}{t_n^2}.
\end{align}
In practice, $\boldsymbol{x}_0$ will be predicted via a neural network and thus the reverse conditional is defined as $p_{\boldsymbol{\theta}}(\boldsymbol{x}_{n-1}|\boldsymbol{x}_n)=p(\boldsymbol{x}_{n-1}|\boldsymbol{x}_n,\hat{\boldsymbol{x}}_0(\boldsymbol{x}_n, t_n;\boldsymbol{\theta}))$.
Since both forward and reverse conditionals are Gaussian with closed‐form parameters, the joint density ratio over trajectories can be evaluated exactly with minimal overhead.

\section{VT-DIS for E(3)-Equivariant Diffusion Models}
\label{app:vt-dis-e3}

In this appendix we detail how the proposed velocity-tempered diffusion importance
sampling (VT-DIS) post-training procedure can be integrated
with $\mathrm{E}(3)$-equivariant diffusion models.

\paragraph{Notation.}
Let $M \ge 2$ denote the number of particles and let $n$ be the spatial
dimension.  Concatenate the coordinates of all particles into the column vector
\[
    \boldsymbol{x} = \bigl(\boldsymbol{x}^{(1)}, \dots, \boldsymbol{x}^{(M)}\bigr)
    \in \mathbb{R}^{Mn},
    \qquad
    \boldsymbol{x}^{(i)} \in \mathbb{R}^{n}.
\]
Unless stated otherwise we restrict our attention to the zero–centre-of-mass
(CoM) subspace
\[
    \mathcal{X}_0
    = \Bigl\{ \boldsymbol{x} \in \mathbb{R}^{Mn} :
        \tfrac{1}{M}\sum_{i=1}^{M} \boldsymbol{x}^{(i)} = \boldsymbol{0}
      \Bigr\}.
\]

\paragraph{Change-of-basis matrix.}
To remove the CoM degree of freedom, we construct a change-of-basis matrix
$P$ that maps any $\boldsymbol{x}\in\mathcal{X}_0$ from the
$Mn$-dimensional ambient space to an $(M-1)n$-dimensional Euclidean space.
Choose an orthonormal matrix
$V_M \in \mathbb{R}^{(M-1)\times M}$
whose rows span the subspace
$\{\boldsymbol{z}\in\mathbb{R}^{M} : \boldsymbol{1}^{\top}\boldsymbol{z}=0\}$;
one explicit choice is obtained by applying a QR
factorization to $I_M - \tfrac{1}{M}\boldsymbol{1}\boldsymbol{1}^{\top}$.
Define
\[
    P := V_M \otimes I_n
    \;\in\; \mathbb{R}^{(M-1)n \times Mn},
\]
where $\otimes$ denotes the Kronecker product.
For any $\boldsymbol{x}\in\mathcal{X}_0$ we set
$\tilde{\boldsymbol{x}} = P\boldsymbol{x}\in\mathbb{R}^{(M-1)n}$.
The identities
\begin{equation}
    PP^{\top} = I_{(M-1)n},
    \qquad
    P^{\top}P
      = \bigl(I_M - \tfrac{1}{M}\boldsymbol{1}\boldsymbol{1}^{\top}\bigr)
        \otimes I_n,
    \label{eq:P-identities}
\end{equation}
follow immediately from $V_M V_M^{\top} = I_{M-1}$ and
$V_M^{\top}V_M = I_M - \tfrac{1}{M}\boldsymbol{1}\boldsymbol{1}^{\top}$.
The first identity shows that $P$ is an isometry between
$\mathcal{X}_0$ and $\mathbb{R}^{(M-1)n}$.

\paragraph{Evaluating and Sampling from $\mathcal{N}_x$ with full covariance}
For $\boldsymbol\mu\in\mathcal X_0$ and a positive-definite $\Sigma\in\mathbb R^{Mn\times Mn}$ whose image lies in $\mathcal X_0$, define the density
\begin{align}
\mathcal N_{\!x}(\boldsymbol x;\,\boldsymbol\mu,\Sigma)
&= \mathcal{N}(\tilde{\boldsymbol{x}}; \tilde{\boldsymbol{\mu}}, P\Sigma P^T)\\
\label{equation: com gaussian}
  &=\;\frac{1}{(2\pi)^{\frac{(M-1)n}{2}}\sqrt{\det(P\Sigma P^{\top})}}
        \exp\!\bigl[-\tfrac12(\boldsymbol x-\boldsymbol\mu)^{\top}P^{\top}P\,\Sigma^{-1}P^{\top}P(\boldsymbol x-\boldsymbol\mu)\bigr].
\end{align}
To get samples from this distribution, we have to first sample from the normal distribution with dimension $(M-1)n$ and then we map it back to the $Mn$ dimensional ambient space so that its center of gravity equals zero using the change-of-basis matrix $P$. This is the general way to get samples for $\mathcal{N}_x$ with full covariance. Things will simplify if the covariance becomes isotropic.

If $\Sigma=\sigma^{2}I_{Mn}$ is isotropic,~\eqref{equation: com gaussian} will becomes:
\begin{equation*}
\mathcal N_{\!x}(\boldsymbol x;\,\boldsymbol\mu,\sigma^2\boldsymbol{I})
  \;:=\;\frac{1}{(2\pi\sigma^2)^{\frac{(M-1)n}{2}}}
        \exp\!\bigl[-\frac{1}{2\sigma^2}(\boldsymbol x-\boldsymbol\mu)^{\top}(\boldsymbol x-\boldsymbol\mu)\bigr].
\end{equation*}
because of the fact that $PP^\top=I_{M-1}$. This is the Gaussian distribution constrained on zero CoM subspace with isotropic covariance \citep{hoogeboom2022equivariant, xu2022geodiff}. To sample from this distribution, we could sample in the $Mn$ dimensional ambient space directly, and subtract its center of mass \citep{xu2022geodiff, satorras2021n}.

\subsection{Proposed Covariance Structure Satisfies \texorpdfstring{$\mathrm{E}(3)$}{E(3)}-Equivariance}

A diffusion model is \(\mathrm{E}(3)\)-equivariant if, for every
orthogonal transformation
\(R\in\mathrm{O}(Mn)\),\footnote{%
  Translations are excluded because we restrict to the zero–CoM
  subspace~\(\mathcal{X}_0\); cf.\ Sec.~\ref{app:vt-dis-e3}.}
(i) the forward-process endpoint distribution
\(p(\boldsymbol{x}_T)\) is invariant,
and (ii) the reverse-process conditionals
\(p(\boldsymbol{x}_{t-1}\mid\boldsymbol{x}_t)\) are equivariant:
\[
  p(R\boldsymbol{x}_{t-1}\mid R\boldsymbol{x}_t)
  \;=\;
  p(\boldsymbol{x}_{t-1}\mid\boldsymbol{x}_t)
  \quad\forall\,t\in\{1,\dots,T\}.
\]
Invariance of the isotropic Gaussian prior
\(p(\boldsymbol{x}_T)=\mathcal{N}_x(\boldsymbol{0},T^2I_{Mn})\)
is immediate.
We therefore focus on the conditional law.

\begin{lemma}
  Let the reverse transition be
  \[
    p(\boldsymbol{x}_{t-1}\mid\boldsymbol{x}_t)
    \;=\;
    \mathcal{N}_x\!\bigl(\boldsymbol{x}_{t-1};
      \boldsymbol{x}_t,\,
      \Sigma\bigr),
    \qquad
    \Sigma \;=\; B\otimes I_n,
  \]
  where \(B\in\mathbb{R}^{M\times M}\) is positive definite.
  Then the reverse kernel is \(\mathrm{E}(3)\)-equivariant.
\end{lemma}
\begin{proof}
  Every spatial rotation acts identically on all $M$ particles and can
  therefore be written as the block-orthogonal matrix
  \[
      R \;=\; I_M \otimes R_n,
      \qquad
      R_n \in \mathrm{O}(n).
  \]
  Recall that the reverse kernel is defined on the \((M-1)n\)-dimensional
  coordinate $\tilde{\boldsymbol{x}} := P\boldsymbol{x}$ via
  \[
      p(\boldsymbol{x}_{t-1}\mid\boldsymbol{x}_t)
      \;=\;
      \mathcal{N}\!\bigl(P\boldsymbol{x}_{t-1};
        P\boldsymbol{x}_t,\,
        P\Sigma P^T\bigr),
      \qquad
      \Sigma \;=\; B\otimes I_n,
      \quad
      B \succ 0.
  \]
  For arbitrary $\boldsymbol{x},\boldsymbol{y}\in\mathcal{X}_0$ we have
  \begin{align*}
    \log p(R\boldsymbol{y}\mid R\boldsymbol{x})
    &= \log\mathcal{N}(PR\boldsymbol{y}; PR\boldsymbol{x}, P\Sigma P^\top)\\
      &= -\tfrac12\!
         \bigl(PR\boldsymbol{y} - PR\boldsymbol{x}\bigr)^{\!\top}
         P\Sigma^{-1}P^\top
         \bigl(PR\boldsymbol{y} - PR\boldsymbol{x}\bigr)
         \;+\; \text{const} \\[4pt]
      &= -\tfrac12\!
         (\boldsymbol{y}-\boldsymbol{x})^{\!\top}R^\top P^\top P
         \Sigma^{-1}P^\top PR
         (\boldsymbol{y}-\boldsymbol{x}) \\[4pt]
      &\stackrel{(a)}{=} -\tfrac12\!
         (\boldsymbol{y}-\boldsymbol{x})^{\!\top}
         P^\top PR^\top
         \Sigma^{-1}
         RP^\top P
         (\boldsymbol{y}-\boldsymbol{x}) \\[4pt]
      &\stackrel{(b)}{=}
         -\tfrac12\!
         (\boldsymbol{y}-\boldsymbol{x})^{\!\top}
         P^\top P
         \Sigma^{-1}
         P^\top P
         (\boldsymbol{y}-\boldsymbol{x}) \\[4pt]
      &= -\tfrac12\!
         \bigl(P\boldsymbol{y}-P\boldsymbol{x}\bigr)^{\!\top}
         P\Sigma^{-1}P^\top
         \bigl(P\boldsymbol{y}-P\boldsymbol{x}\bigr) \\[2pt]
      &= \log\mathcal{N}(P\boldsymbol{y}; P\boldsymbol{x}, P\Sigma P^\top)\\
      &= \log p(\boldsymbol{y}\mid\boldsymbol{x}),
  \end{align*}
where step~$(a)$ uses the fact that \(P^{\top}P\) commutes with $R$:
\begin{align*}
RP^\top P = (I_n\otimes R_n)(V_M\otimes I_n) = V_M\otimes R_n = P^\top PR
\end{align*}
and step~$(b)$ uses the fact that
\begin{align*}
R^\top \Sigma^{-1}R &= (R^\top\Sigma R)^{-1}\\
&= ((I_M\otimes R_n^\top)(B\otimes I_n)(I_M\otimes R_n))^{-1}\\
&= (B\otimes R_n^\top R_n)^{-1}\\
&= (B\otimes I_n)^{-1}\\
&= \Sigma^{-1}
\end{align*}
Hence we have proved that our proposed covariance form satisfies the $\mathrm{E}(3)$-equivariance required.
\end{proof}

\subsection{Permutation Equivariance for Many-Particle Systems}
\label{section: permutation equivariance many particle systems}
Datasets such as \textsc{DW-4} and \textsc{LJ-13} treat all
particles as indistinguishable; hence the reverse kernel must be equivariant
under any permutation of the $M$ particles.

\begin{lemma}[Permutation equivariance]
  Let the covariance in the reduced coordinate space be
  \[
      \Sigma \;=\; B\otimes I_n,
      \qquad
      B \;=\; (b-a)\,I_{M-1} \;+\; a\,\boldsymbol{1}\boldsymbol{1}^{\!\top},
  \]
  with parameters \(b>0\) and \(a\in\mathbb{R}\) chosen such that
  \(
    b-a>0
  \)
  and
  \(
    b+(M-2)a>0,
  \)
  ensuring \(B\succ0\).
  For any permutation matrix
  \(S_M\in\mathbb{R}^{M\times M}\) and the block permutation
  \(S := S_M\otimes I_n\),
  the reverse transition satisfies
  \(
      p(S\boldsymbol{y}\mid S\boldsymbol{x})
      = p(\boldsymbol{y}\mid\boldsymbol{x})
  \)
  for all
  \(\boldsymbol{x},\boldsymbol{y}\in\mathcal{X}_0\).
\end{lemma}
\begin{proof}
  \textbf{(i) Commutation with the CoM projector.}
  Using \(P^{\top}P = (V_M\otimes I_n)\) (Sec.~\ref{app:vt-dis-e3}) and the
  fact that \(S_M \boldsymbol{1} = \boldsymbol{1}\) and $\boldsymbol{1}^\top S_M=\boldsymbol{1}^\top$,
  \[
      SP^{\top}P
      = \bigl(S_M\bigl(I_M-\frac{1}{M}\boldsymbol{1}\boldsymbol{1}^\top\bigr)\bigr)\otimes I_n
      = \bigl(S_M-\frac{1}{M}\boldsymbol{1}\boldsymbol{1}^\top\bigr)\otimes I_n
      = P^{\top}PS,
  \]
  so \(S\) commutes with the projector onto the zero–CoM subspace.

  \textbf{(ii) Invariance of the covariance.}
  Because \(S_M\) merely reorders rows and columns,
  \[
      S_M B S_M^{\!\top}
      = (b-a)\,I_{M-1} + a\,S_M\boldsymbol{1}\boldsymbol{1}^{\!\top}S_M^{\!\top}
      = (b-a)\,I_{M-1} + a\,\boldsymbol{1}\boldsymbol{1}^{\!\top}
      = B,
  \]
  whence
  \(
      S\Sigma S^{\!\top}
      = \Sigma.
  \)

  \textbf{(iii) Equality of densities.}
  With the same quadratic-form argument used in the
  $\mathrm{E}(3)$-equivariance proof,
  \[
      (\boldsymbol{y}-\boldsymbol{x})^{\!\top}P^{\top}P\,\Sigma^{-1}\,P^{\top}P
      (\boldsymbol{y}-\boldsymbol{x})
      \;=\;
      (S\boldsymbol{y}-S\boldsymbol{x})^{\!\top}
      P^{\top}P\,\Sigma^{-1}\,P^{\top}P
      (S\boldsymbol{y}-S\boldsymbol{x}),
  \]
  establishing
  \(p(S\boldsymbol{y}\mid S\boldsymbol{x})
    = p(\boldsymbol{y}\mid\boldsymbol{x})\).
\end{proof}

\subsection{Permutation Equivariance for General Molecular Distributions}
\label{appendix: permutation equivariance for general molecule distributions}
Again, we aim to show that our choice of $B_{\boldsymbol{\phi}}$ satisfies the permutation-equivariance constraint under the parametrization proposed in Section~\ref{section: e3 diffusion with importance sampling}.

\textbf{Permutation equivariance.}\;
Let $S_M$ denote any permutation matrix that reorders atoms \emph{within} each
label class (i.e., $L_{S_M(i)} = L_i$).
Analogous to the proof in Section~\ref{section: permutation equivariance many particle systems}, because any permutation matrix commutes with the CoM projector, this condition is automatically satisfied.  
All that remains is to prove the invariance of the covariance.  
Because the two parameterizations above depend on the indices only through the labels, we have
\[
    S_M\,B_{\boldsymbol{\phi}}(n)\,S_M^{\!\top}=B_{\boldsymbol{\phi}}(n).
\]
Consequently, the lifted covariance
\[
    \Sigma_{\boldsymbol{\phi}}(n)=B_{\boldsymbol{\phi}}(n)\otimes I_n
\]
satisfies
\[
    S\,\Sigma_{\boldsymbol{\phi}}(n)\,S^{\!\top}=\Sigma_{\boldsymbol{\phi}}(n)
\]
for every block permutation
\(
    S = S_M \otimes I_n,
\)
so the reverse kernel remains permutation-equivariant by the argument of the above section.

\section{Implementation Details}
\label{appendix:implementation-details}

All experiments were run on a single NVIDIA A100 GPU (80 GB); hence every “GPU hour’’ refers to wall-clock hours on that device.

\subsection{Implementation of the objective function}
\label{section: implementation of objective}
Recall the objective (the $\alpha$-divergence with $\alpha=2$):
\[
\mathcal{L}(\boldsymbol{\phi})
=\mathbb{E}_{q}\!\left[\frac{q(\mathbf{x}_{0:N})}{p_{\boldsymbol{\phi}}(\mathbf{x}_{0:N})}\right]
\approx \frac{1}{M}\sum_{m=1}^M \frac{q(\mathbf{x}_{0:N}^{(m)})}{p_{\boldsymbol{\phi}}(\mathbf{x}_{0:N}^{(m)})}
= \frac{1}{M}\sum_{m=1}^M w^{(m)},
\]
where $\mathbf{x}^{(m)}_{0:N}\sim q$, $w^{(m)}:=\frac{q(\mathbf{x}^{(m)}_{0:N})}{p_{\boldsymbol{\phi}}(\mathbf{x}^{(m)}_{0:N})}$, and $M$ is the number of samples used to estimate the objective.

In practice, we optimize the \emph{log} of the $\alpha$-divergence for numerical stability, which can be computed efficiently using the log-importance weights via the \texttt{logsumexp} trick. Specifically,
\begin{equation}
\label{equation: log alpha divergence}
\mathcal{J}(\boldsymbol{\phi})
:=\log \mathcal{L}(\boldsymbol{\phi})
\approx \log\!\left(\frac{1}{M}\sum_{m=1}^M \exp(\log w^{(m)})\right)
= \texttt{logsumexp}\!\big(\{\log w^{(m)}\}_{m=1}^M\big) - \log M .
\end{equation}
Hence, we work entirely in log space, which yields a numerically stable estimator of the loss.

\subsection{Datasets}
\begin{itemize}[leftmargin=*]
\item \textbf{GMM-2.}  
      A two-component Gaussian mixture used as a toy benchmark.
      Training and evaluation samples are drawn directly from the target
      distribution.

\item \textbf{DW-4 and LJ-13.}  
      We use the $10^{7}$ configurations released by \citet{klein2023equivariant}.  
      A subset of $10^{5}$ configurations trains the score model and VT-DIS;
      the remainder forms the test set.

\item \textbf{Alanine dipeptide.}  
      We adopt the molecular dataset of
      \citet{midgley2023se}, which provides $10^{5}$ training samples and
      $10^{6}$ test samples.
\end{itemize}

\subsection{Training the Score Model}
Our training pipeline—including parameterization and preconditioning—follows \citet{karras2022elucidating}.

\begin{itemize}[leftmargin=*]
\item \textbf{GMM-2.}  
      A seven-layer MLP (hidden width 512), identical to the network used by \citet{akhound2024iterated} for the GMM-40 task.  
      Training runs for 100,000 iterations with a batch size of 1,024 for both \(d=50\) and \(d=100\).

\item \textbf{DW-4.}  
      EGNN backbone with 4 layers and a hidden width of 128.

\item \textbf{LJ-13.}  
      EGNN backbone with 6 layers and a hidden width of 128.

\item \textbf{Alanine Dipeptide.}  
      EGNN backbone with 7 layers and a hidden width of 256; node features include atom types and bond information.  
      Scaling atomic coordinates by 10 during training improves sample quality; network outputs are divided by 10 at inference.
\end{itemize}

All models are trained for 100,000 iterations with a batch size of 512 using the Adam optimizer, a learning rate of 0.001, and a cosine annealing schedule with a minimum learning rate of \(1\times10^{-6}\). Table~\ref{tab:score-training-time} summarizes these details.

\begin{table}[t]
\centering
\caption{Score-model training time on a single A100.}
\label{tab:score-training-time}
\begin{tabular}{@{}lcccccc@{}}
\toprule
Dataset   & Architecture             & Layers & Hidden dim & Batch Size & Iterations & GPU-h \\ \midrule
GMM-2 ($d{=}50$) & MLP      & 7      & 512        & 5000  & 100 000    & 0.6     \\
GMM-2 ($d{=}100$) & MLP      & 7      & 512        & 5000  & 100 000    & 1.0     \\
DW-4         & EGNN          & 4      & 128        & 512   & 100 000    & 0.5     \\
LJ-13        & EGNN          & 6      & 128        & 512   & 100 000    & 8.0     \\
Alanine dipeptide  & EGNN    & 7      & 256        & 512   & 100 000    & 8.0     \\ \bottomrule
\end{tabular}
\end{table}

\subsection{VT-DIS Post-training}
\label{subsec:vtdis_posttraining}

We follow Algorithm~\ref{alg:vt-dis} exactly. With the score network frozen, we optimize the per-step covariance matrices by minimizing the \(\alpha\)-divergence (\(\alpha=2\)) between the forward and reverse processes using the same training set as for score learning. Each run comprises 5,000 optimization steps with a batch size of 512 (256 for alanine dipeptide). All runs use the Adam optimizer with a learning rate of 0.01 and a cosine annealing schedule with a minimum learning rate of \(1\times10^{-6}\). Table~\ref{tab:vtdis-training-time} summarizes these settings.

The reported training time corresponds to tuning the isotropic covariance; diagonal and full covariance forms incur similar runtimes. We fix the number of time steps to 100; empirically, with fixed epochs and batch sizes, training time increases linearly with the number of sampling steps.

\begin{table}[t]
\centering
\caption{VT-DIS post-training time on a single A100.}
\label{tab:vtdis-training-time}
\begin{tabular}{@{}lcccc@{}}
\toprule
Dataset           & Batch Size & Iterations & Steps $K$ & GPU-h \\ \midrule
GMM-2 $(d=50)$             & 512   & 5 000      & 100        & 1.8     \\
GMM-2 $(d=100)$            & 512   & 5 000      & 100        & 3.5     \\
DW-4              & 512   & 5 000      & 100        & 0.9     \\
LJ-13             & 512   & 5 000      & 100        & 2.7     \\
Alanine dipeptide & 128   & 5 000      & 100        & 4.3     \\ \bottomrule
\end{tabular}
\end{table}

\subsection{Evaluation Metrics}
\label{subsec:evaluation_metrics}

We measure efficiency using forward and reverse effective sample sizes (ESS), defined in Eqs.~\eqref{eq:ess-forward} and \eqref{eq:ess-reverse}. The number of Monte Carlo draws for each estimate is specified in the corresponding figure or table caption. Table~\ref{tab:vtdis-sampling-time} summarizes the sampling time for each target distribution. Empirically, sampling time increases linearly with the number of sampling steps.

\begin{table}[t]
\centering
\caption{VT-DIS sampling time on a single A100.}
\label{tab:vtdis-sampling-time}
\begin{tabular}{@{}lccc@{}}
\toprule
Dataset           & Batch Size & Steps $K$ & GPU-min \\ \midrule
GMM-2 $(d=50)$    & 100,000   & 100      & 0.25\\
GMM-2 $(d=100)$   & 100,000   & 100      & 0.5\\
DW-4              & 100,000   & 100      & 1.3\\
LJ-13             & 100,000   & 100      & 4.3\\
Alanine dipeptide & 10,000   & 100      & 3 \\ 
\bottomrule
\end{tabular}
\end{table}

\section{Additional Experimental Results}
\subsection{Computational Cost for Solving Probability Flow ODE}
\label{section: extended computational cost}
In this section, we compare effective sample size (ESS) and inference‑time computational cost for an additional baseline: PF‑ODE with JVP‑based trace estimation. Table~\ref{tab:dw4lj13_pf_ode_extended} summarizes the results. Using Hutchinson’s trace estimator reduces computational cost; however, with a single probe vector the ESS drops sharply, whereas using 10 probe vectors increases both cost and ESS. VT‑DIS retains the lowest inference‑time cost among all baselines considered. We emphasize that although Hutchinson’s estimator is unbiased for the instantaneous divergence, using stochastic divergence estimates to compute importance weights can introduce bias in self‑normalized IS estimates. Consistent with prior work on these datasets (e.g., \citep{klein2023equivariant, tan2025scalable}), our main‑text comparison computes exact divergences when applying IS to avoid this bias. Accordingly, we caution against using stochastic trace estimators for IS and explicitly flag this caveat in the PF‑ODE+trace results.

\begin{table}[t]
  \centering
  \begin{threeparttable}
  \caption{Performance of importance sampling with diffusion on DW-4, LJ-13 and ALDP with $10^{4}$ samples on a single A100 GPU. The number in parentheses after each dataset denotes the sampling steps. Methods: reverse SDE (VT-DIS) or PF ODE (exact or Hutchinson with $n$ probe vectors). $\uparrow$ = higher is better, $\downarrow$ = lower is better.}
  \label{tab:dw4lj13_pf_ode_extended}
  \begin{tabular}{@{}l l S[table-format=3.2] S[table-format=3.2] S[table-format=5.2]@{}}
    \toprule
    \textbf{Dataset (steps)} & \textbf{Method} &
      {\textbf{Forward ESS (\%)} $\uparrow$} &
      {\textbf{Reverse ESS (\%)} $\uparrow$} &
      {\textbf{GFLOP / sample} $\downarrow$} \\
    \midrule
    \multirow{4}{*}{\textbf{DW-4} (100)}
      & SDE (Ours)                     & 44.32 & 46.76 & 0.81 \\
      & ODE (Exact)                       & 94.29 & 83.60 & 7.70 \\
      & ODE (Est. \ensuremath{n{=}1})\tnote{\dag}
                                             & 37.70 & 0.95 & 1.56 \\
      & ODE (Est. \ensuremath{n{=}10})\tnote{\dag}
                                             & 80.60 & 82.67 & 8.27 \\
    \midrule
    \multirow{4}{*}{\textbf{LJ-13} (200)}
      & SDE (Ours)                     & 15.87 & 17.69 & 26.20 \\
      & ODE (Exact)                       & 79.21 & 77.42 & 2173.00 \\
      & ODE (Est. \ensuremath{n{=}1})\tnote{\dag}
                                             & 12.35 & 9.50 & 52.26 \\
      & ODE (Est. \ensuremath{n{=}10})\tnote{\dag}
                                             & 61.12 & 66.10 & 285.62 \\
    \midrule
    \multirow{4}{*}{\textbf{ALDP} (200)}
      & SDE (Ours)                     & 0.00 & 1.05 & 33.70 \\
      & ODE (Exact)                       & 0.00 & 35.39 & 45967.00 \\
      & ODE (Est. \ensuremath{n{=}1})\tnote{\dag}
                                             & 0.00 & 0.74 & 702.78 \\
      & ODE (Est. \ensuremath{n{=}10})\tnote{\dag}
                                             & 0.00 & 22.31 & 3850.39 \\
    \bottomrule
  \end{tabular}
  \begin{tablenotes}\footnotesize
    \item[\dag] Hutchinson trace estimator with \(n\) probe vectors; introduces bias in importance sampling and can misstate ESS.
  \end{tablenotes}
  \end{threeparttable}
\end{table}

\subsection{Optimized Covariance Matrices}
\label{appendix: add experiments optimized cov}
\begin{figure}[t]
  \centering
  \begin{subfigure}[b]{0.48\linewidth}
    \centering
    \includegraphics[width=\linewidth]{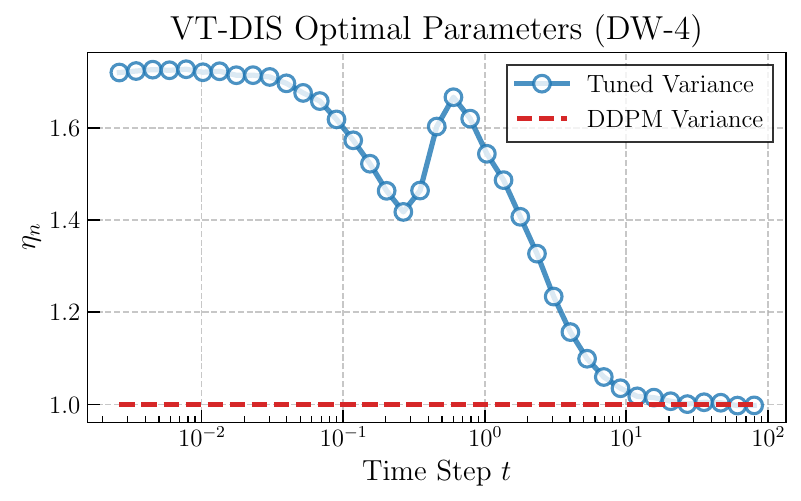}
    \caption{DW-4 (40 steps)}
    \label{fig:tuned_dw4}
  \end{subfigure}
  \hfill
  \begin{subfigure}[b]{0.48\linewidth}
    \centering
    \includegraphics[width=\linewidth]{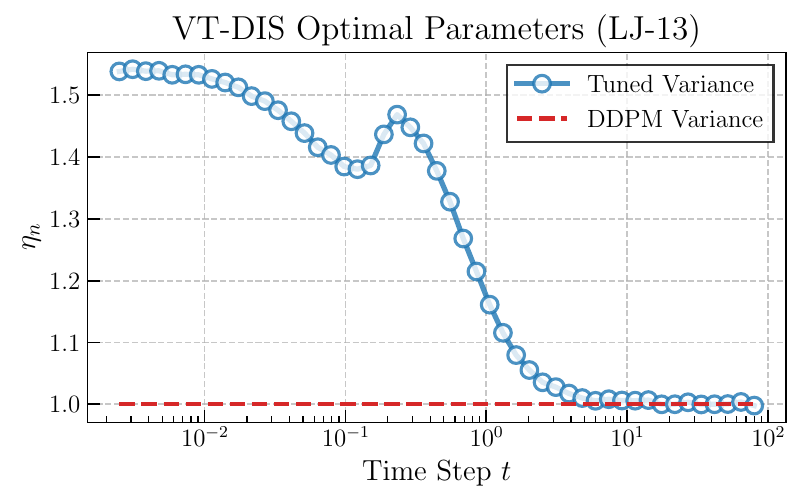}
    \caption{LJ-13 (50 steps)}
    \label{fig:tuned_lj13}
  \end{subfigure}
  \caption{\textbf{Optimal variance scaling \(\eta_n\).}
  VT-DIS inflates early-time variance to improve coverage,
  adapts non-monotonically across the transition regime,
  and converges to the DDPM value (\(\eta_n=1\), red dashed line)
  in the pure-noise limit.}
  \label{fig:tuned_parameters}
\end{figure}
\label{appendix: additional experiment discussion}
Figure~\ref{fig:tuned_parameters} contrasts the per–step
variance–scaling factor of VT-DIS,
\(\eta_n\) (see \eqref{equation: isotropic parameterization}),
with the unit baseline of the original DDPM.
Although we display the isotropic case, the same qualitative pattern
appears for diagonal and low–rank covariances.

\paragraph{Three regimes.}
Reading the \(t\)-axis from left (\(t\!\approx\!10^{-3}\)) to right
(\(t\!\approx\!10^{2}\)) follows the
reverse SDE from data space to pure noise.
The learned profile is \emph{not} monotone but exhibits
three characteristic regimes:

\begin{enumerate}[leftmargin=1.5em,itemsep=0.25em]
\item \textbf{Data–dominated regime} (\(t\lesssim 10^{-1}\)).  
      Here the signal-to-noise ratio (SNR) is high,
      the score network is the main source of model error,
      and under-dispersion relative to the true forward kernel
      results in heavy importance-weight tails.
      VT-DIS compensates by \emph{inflating} the variance
      (\(\eta_n\approx 1.6\text{–}1.7\)),
      which widens the proposal and improves coverage.

\item \textbf{Transition regime} (\(10^{-1}\lesssim t \lesssim 10^{1}\)).  
    As the SNR drops, the target conditional gradually morphs from highly non-Gaussian to approximately Gaussian. The mismatch between the analytic DDPM variance (optimal for the \emph{Gaussian} limit) and the true kernel therefore decreases. VT-DIS first reduces the scaling (a local minimum at \(t\!\approx\!0.3\)), then overshoots to a peak at \(t\!\approx\!1\), before descending. The intermediate bump arises where the learned score \(D_{\boldsymbol{\theta}}\) tends to \emph{over-estimate} the gradient magnitude (a common artefact of score matching when the SNR is intermediate, e.g. \(\mathcal{O}(1)\)) \citep{xu2023stable}. Increasing the proposal variance temporarily counter-balances this over-confidence, reducing the second-moment of the importance weights.

\item \textbf{Noise regime} (\(t\gtrsim 10^{1}\)).  
      The state distribution is now close to \(\mathcal{N}(0,t^2I)\),
      for which the DDPM variance is \emph{exactly} optimal.
      Accordingly, \(\eta_n\) relaxes to~1,
      confirming that VT-DIS does not perturb steps that are already
      statistically efficient.
\end{enumerate}

\paragraph{Implications.}
The variance inflations at early steps are crucial:
without them, trajectories started near the data manifold
place negligible mass in regions where the forward kernel
has significant probability, causing weight explosion
and the near-zero forward ESS observed for the vanilla baseline.
Conversely, keeping \(\eta_n\!\to\!1\) in the noise regime
ensures that VT-DIS retains the favourable properties of the original DDPM sampler at large time steps, adding no unnecessary dispersion..

\subsection{KL versus \texorpdfstring{$\alpha$}{alpha}-divergence}
\label{subsubsec:kl_vs_alpha}

The VT-DIS objective in \eqref{eq: objective function} minimises the \(\alpha\)-divergence with
\(\alpha=2\),
which is proportional to the \emph{variance} of importance weights and therefore maximizes ESS directly \citep{midgley2022flow}.
For completeness we retrain the covariance parameters on \textsc{DW-4}
using the forward KL divergence
(\(\alpha=1\)) and compare the resulting samplers.

Figure~\ref{fig:ess_alpha_comparison} plots forward (solid) and reverse
(dashed) ESS as functions of the sample size \(M\).
Tuning with \(\alpha=2\) yields consistently higher and \emph{stable}
ESS in both directions.
In contrast, KL-tuned parameters give (i) low forward ESS and
(ii) reverse ESS that \emph{decreases} with \(M\),
a classic symptom of heavy weight tails caused by poor proposal coverage.
These observations confirm that controlling the
second moment of weights (\(\alpha=2\)) is crucial for robust importance
sampling, whereas KL divergence alone fails to regularise the tails.

\begin{figure}[t]
  \centering
  \includegraphics[width=0.58\linewidth]{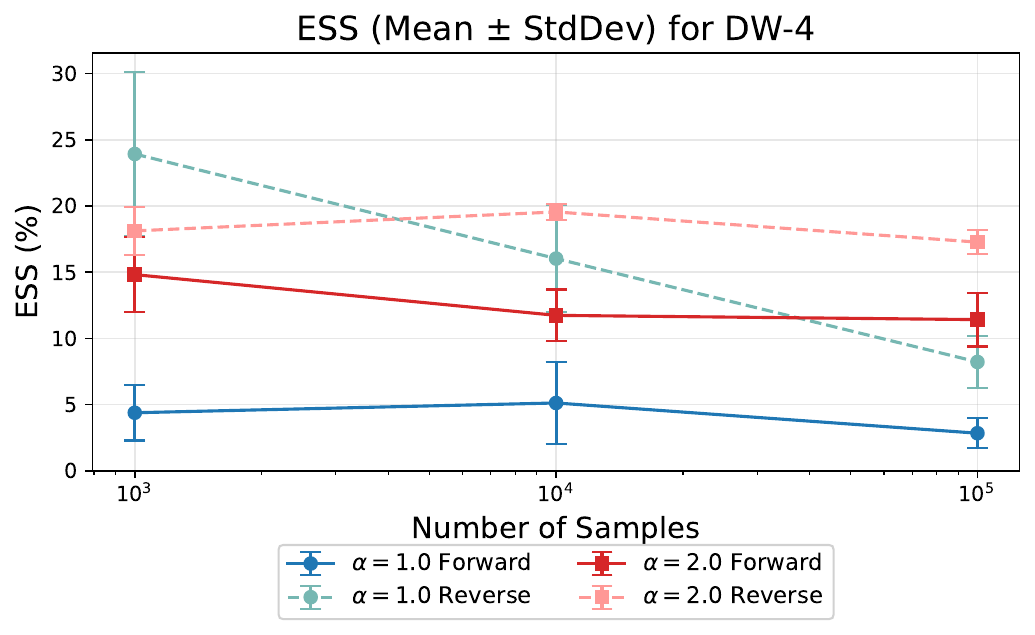}
  \caption{\textbf{KL (\(\alpha=1\)) vs.\ \(\alpha=2\) divergence on \textsc{DW-4}.}
           \(\alpha=2\) tuning (red) attains higher, size-independent ESS,
           whereas KL tuning (blue) collapses for large \(M\).}
  \label{fig:ess_alpha_comparison}
\end{figure}

\section{Broader Impacts}
\label{appendix: broader impacts}
Our method enhances the efficiency and accessibility of unbiased molecular sampling, potentially speeding up the discovery of catalysts, drugs, and new materials while reducing reliance on costly experiments. At the same time, stronger samplers could be misused to design toxic or environmentally harmful compounds more rapidly. There is also a risk that institutions with advanced computing resources will widen the gap with those lacking such infrastructure, reinforcing inequities in scientific research.

\end{document}